\documentclass{article}

\usepackage[main, final]{neurips_2026}

\usepackage[utf8]{inputenc} 
\usepackage[T1]{fontenc}    
\usepackage{hyperref}       
\usepackage{url}            
\usepackage{booktabs}       
\usepackage{amsfonts}       
\usepackage{nicefrac}       
\usepackage{graphicx}
\usepackage{microtype}      
\usepackage{xcolor} 
\usepackage{wrapfig}
\usepackage{amsthm}
\usepackage{amsmath}
\usepackage{pifont}
\theoremstyle{plain}
\newtheorem{theorem}{Theorem}[section]

\newtheorem{proposition}[theorem]{Proposition}
\newtheorem{corollary}[theorem]{Corollary}

\theoremstyle{definition}
\newtheorem{definition}[theorem]{Definition}
\newtheorem{axiom}[theorem]{Axiom}

\theoremstyle{remark}

\begin{document}

\title{Relations Are Channels: Knowledge Graph Embedding via Kraus Decompositions}

%

\author{%
  Sayan K Chaki \\
  Inria, Laboratoire Hubert Curien, Université Jean Monnet\\
  \texttt{sayan.chaki@inria.fr } \\
}

\maketitle

\begin{abstract}
Knowledge graph embedding (KGE) models typically represent each relation as an operator on entity embeddings. In this work, we identify three structural axioms that any principled relation operator must satisfy, linearity, trace preservation, and complete positivity, and show that they characterize a Kraus channel structure via the Kraus representation theorem. The completeness constraint defining this family is equivalent to these axioms, providing a principled foundation rather than an externally imposed condition. Under this formulation, most existing operator-based KGE models are recoverable as special cases with Kraus rank $\kappa = 1$ under specific embedding choices. We further generalize this characterization to arbitrary metric geometries by introducing \mbox{w-Kraus} channels, which satisfy completeness by construction within their respective spaces. Building on this theory, we propose \textsc{KrausKGE}, a principled KGE model that naturally handles $1$-to-$N$ and $N$-to-$N$ relations, supports $k$-hop reasoning without requiring explicit path encoders, and eliminates the need for norm constraints on entity embeddings. Additionally, our framework yields the first theoretically grounded per-relation complexity measure in the KGE literature, with a provable lower bound in terms of the empirical relation matrix rank. Empirical evaluation demonstrates that \textsc{KrausKGE} consistently outperforms strong baselines on $N$-to-$N$ relations, with performance gains that increase monotonically with relation fan-out, in alignment with theoretical predictions.
\end{abstract}

\section{Introduction}
\label{sec:intro}
Knowledge Graphs (KGs) store factual knowledge as triples 
of the form (\textit{head entity, relation, tail entity}): 
for example, (\textit{Beethoven, Genre, Classical}) encodes 
a factual relation between two entities. Large-scale KGs 
such as YAGO~\cite{suchanek2007yago}, 
Wikidata~\cite{vrandevcic2014wikidata}, and 
Freebase~\cite{bollacker2008freebase} have found broad 
application in question answering~\cite{skandan2023question, 
wang2025llm}, recommendation~\cite{li2023survey, 
wang2024application, wang2025reinforced}, and information 
retrieval~\cite{cao2024knowledge}. However, their large 
scale and complex relational structure make them difficult 
to manipulate directly. Knowledge Graph Embedding (KGE) 
addresses this by learning continuous representations of 
entities and relations, enabling downstream tasks such as 
link prediction~\cite{rossi2021knowledge} and entity 
classification~\cite{wilcke2020end}.

Existing KGE methods organise into four families, each 
addressing a genuine limitation of the others while 
introducing its own. The first encodes each relation as 
a single explicit operator: a translation in 
TransE~\cite{bordes2013translating}, a rotation in 
RotatE~\cite{sun2019rotate}, or an orthogonal transform 
in GoldE~\cite{li2024generalizing} and 
OrthogonalE~\cite{zhu2024block}; these models are 
efficient but collapse multiple valid tails to a single 
point, failing on N-to-N relations. The second uses 
multiple operators per relation for expressiveness, as 
in NTN~\cite{socher2013reasoning} and MEIM~\cite{tran2022meim}, 
but adds operators as an engineering choice with no 
constraint ensuring they are information-preserving. The 
third replaces explicit operators with nonlinear neural 
architectures, as in ConvE~\cite{dettmers2018convolutional} 
and R-GCN~\cite{sheikh2021knowledge}, recovering 
multi-pathway expressiveness at the cost of composability 
for multi-hop reasoning. The fourth represents entities 
as probability distributions, as in 
KG2E~\cite{he2015learning} and 
TransG~\cite{xiao2016transg}, encoding uncertainty without 
a principled account of how it propagates through a 
relation. All four families share the same foundational 
gap: none asks what mathematical structure a relation 
operator must have by necessity, nor derives that 
structure from first principles.

Our work addresses these theoretical and technical gaps. By the Kraus representation theorem\cite{kraus1983states}, a linear map on density matrices is completely positive and trace preserving if and only if it admits a Kraus decomposition. Our contribution is identifying that linearity, trace preservation, and complete positivity are precisely the structural conditions any principled relation operator must satisfy in the KGE setting. Each is natural: a relation operator should be linear on entity representations, it should neither inflate nor deflate the total mass of an entity, and it should behave consistently whether an entity is considered in isolation or as part of a larger correlated structure. Together, these three axioms characterize the Kraus channel family, and the completeness constraint introduced through it is equivalent to their conjunction rather than a soft regulariser added after the fact. We further show that the same axiomatic derivation carries through to different metric geometries defined by a quadratic inner product, yielding w-Kraus channels that satisfy completeness by construction across Euclidean, elliptic, hyperbolic, and product manifold geometries. The geometry is a modelling choice; the axioms are universal.

This identification has three immediate structural consequences, all following as corollaries of completeness with no additional design choices. Multi-hop reasoning becomes architecture-free, norm constraints on entity embeddings become unnecessary, and the learned Kraus rank yields the first theoretically grounded per-relation complexity measure in the KGE literature, with a provable lower bound. The summary of our contributions are as follows:

\begin{itemize}
    \item \textit{Axiomatic characterisation of relation operators.} Three structural axioms on principled relation operators, linearity, trace preservation, and complete positivity. They characterize the Kraus channel structure via the Kraus representation theorem. Most operator-family KGE models are recoverable as special cases with $\kappa = 1$ under specific embedding choices.

    \item \textit{KrausKGE.} A concrete trainable model with density matrix entity representations and learned Kraus operators satisfying completeness by construction, requiring no external regularisation, norm penalties, or path encoders.

    \item \textit{Extension to non-Euclidean spaces.} The derivation generalises to any quadratic inner product geometry, yielding w-Kraus channels with completeness by construction across Euclidean, elliptic, hyperbolic, and product manifold spaces.

    \item \textit{Per-relation complexity measure.} The Kraus rank is the first theoretically grounded to our knowledge per-relation diagnostic in the KGE literature, with provable lower bound.
\end{itemize}
\section{Related work}
\label{sec:related}

The dominant paradigm represents each relation as a single linear operator. 
TransE~\cite{bordes2013translating} encodes relations as translations but is 
structurally unable to model N-to-N relations, collapsing multiple valid tails to a 
single point. TransH~\cite{wang2014knowledge} and TransR~\cite{lin2015learning} 
address specific limitations while retaining a single operator and requiring external 
norm constraints. The bilinear family, comprising RESCAL~\cite{nickel2011three}, 
DistMult~\cite{yang2015embedding}, ComplEx~\cite{trouillon2016complex}, and 
TuckER~\cite{balavzevic2019tucker}, replaces translation with multiplicative 
interaction but remains within the single-operator regime, inheriting the same 
structural bottleneck on high fan-out relations. Rotation-based models 
RotatE~\cite{sun2019rotate}, QuatE~\cite{zhang2019quaternion}, DualE~\cite{cao2021dual}, 
and HousE~\cite{li2022house} achieve strong relational pattern inference through 
geometric operators but are limited to bijective mappings, making N-to-N relations 
representable only approximately through centroid collapse. GoldE~\cite{li2024generalizing} 
and OrthogonalE~\cite{zhu2024block} push this family to state-of-the-art through 
generalised orthogonal parametrisation; \textsc{KrausKGE} strictly generalises both 
by allowing $\kappa > 1$, which Corollary~\ref{coro:single} proves is necessary 
for relations whose empirical matrix rank exceeds the embedding dimension. Geometric 
methods MuRP~\cite{balazevic2019multi}, AttH~\cite{chami2020low}, 
GIE~\cite{cao2022geometry}, and UltraE~\cite{xiong2022ultrahyperbolic} enrich the 
representation space but leave the single-operator bottleneck intact. The survey of 
Cao et al.~\cite{cao2024knowledge} identifies a unified operator framework handling 
all mapping patterns as an open problem; \textsc{KrausKGE} addresses this directly.

A second family replaces explicit operators with nonlinear neural architectures. 
ConvE~\cite{dettmers2018convolutional} applies 2D convolutions to reshaped embeddings; 
R-GCN~\cite{schlichtkrull2018modeling} uses relational graph convolutions; 
KG-BERT~\cite{yao2019kg} scores triples via BERT. These models recover multi-pathway 
expressiveness but abandon operator structure entirely, making multi-hop composition 
intractable without auxiliary encoders. \textsc{KrausKGE} sits between these two 
families: it retains explicit, composable operator structure while matching the 
multi-pathway expressiveness of neural models through concurrent Kraus operators.

A third line deploys multiple operators per relation. NTN~\cite{socher2013reasoning} 
uses $k$ bilinear forms; MEIM~\cite{tran2022meim} uses block term tensor formats for 
strong results at small embedding sizes. Both add operators as a capacity choice with 
no constraint ensuring information is preserved across pathways and no theory 
connecting operator count to relational complexity. The Kraus completeness constraint 
fills both gaps: it guarantees mass preservation across pathways and the rank lower 
bound of Theorem~\ref{thm:rank-bound} connects $\kappa$ directly to the empirical structure 
of each relation.

Multi-hop reasoning has been addressed through RNN-based path encoders 
PTransE~\cite{lin2015modeling} and RSN~\cite{guo2019learning}, reinforcement learning 
agents MultiHopKG~\cite{lin2018multi}, GNN-based encoders A*Net~\cite{zhu2023net} 
and NBFNet~\cite{zhu2021neural}, and sequence-to-sequence generators 
SQUIRE~\cite{bai2022squire}. All correct for the lack of composition closure in their 
base models through auxiliary architectures. Theorem~\ref{thm:composition} shows that 
Kraus channels compose exactly, eliminating the approximation error these encoders 
compensate for. Probabilistic entity models such as KG2E~\cite{he2015learning} and 
DiriE~\cite{wang2022dirie} encode uncertainty but provide no constraint on how it 
propagates through a relation; extending the Kraus framework to stochastic entity 
representations is left to future work.
\section{Kraus Structure of Relation Operators}
\label{sec:theory}

We develop a formal theory that connects relation operators to Kraus channels by identifying key principles for what it means for a relation operator to be
principled. The theory proceeds in four steps: we first establish the correct
representation space for entities, then identify the axioms any relation
operator must satisfy, then prove that those axioms characterise a unique
family of operators. We show in appendix \ref{app:recovery}
 how every existing model and their geometry variants fit as a
special case within that family.

\paragraph{Entity representations as density matrices}

All existing operator-family models represent entities as 
vectors $\mathbf{h}, \mathbf{t} \in \mathbb{R}^d$: natural 
but insufficient. A vector encodes a single point with no 
mechanism to represent uncertainty or multi-faceted 
structure. A film appearing in thousands of triples and a 
rare person appearing in three should not occupy the same 
representational space.

We instead represent each entity as a \emph{density matrix} 
$\rho \in \mathcal{S}(\mathbb{R}^d)$, where
\begin{equation}
    \mathcal{S}(\mathbb{R}^d) \;=\;
    \bigl\{\, \rho \in \mathbb{R}^{d \times d}
    \;\big|\;
    \rho = \rho^\top,\quad \rho \geq 0,\quad 
    \mathrm{Tr}[\rho] \leq 1
    \bigr\}.
    \label{eq:density-space}
\end{equation}
A density matrix encodes a distribution over directions: 
its eigenvectors are the principal axes of the entity and 
its eigenvalues encode probability mass along each. A 
rank-one projector $\rho = \mathbf{v}\mathbf{v}^\top$ 
recovers any existing vector embedding exactly. Beyond 
this, uncertainty is represented naturally through matrix 
rank: well-attested entities converge to low-rank matrices 
during training, rare or ambiguous ones retain higher rank 
with no additional mechanism required. Similarity is 
measured by Hilbert-Schmidt similarity\cite{zyczkowski2006introduction}, which reduces to squared 
cosine similarity for rank-one matrices, recovering 
existing scoring functions as special cases.

\paragraph{Axiomatic Sets}
Having established the entity representation space, we now ask what
constraints any principled relation operator must satisfy. We identify three
axioms. Let $\mathcal{S}(\mathbb{R}^d)$ denote the space of density matrices defined in Equation~(\ref{eq:density-space}), and let $\mathcal{L}^{(r)}: \mathcal{S}(\mathbb{R}^d) \to \mathcal{S}(\mathbb{R}^d)$ be the operator associated with relation $r$, mapping head entity representations to predicted tail representations.

\begin{axiom}(Linearity)
\label{axiom:linearity}
    $\mathcal{L}^{(r)}$ is linear.
\end{axiom}
\begin{axiom}(Trace preservation)
\label{axiom:trace}

For any $\rho \in \mathcal{S}(\mathbb{R}^d)$,
\begin{equation}
    \mathrm{Tr}\bigl[\mathcal{L}^{(r)}(\rho)\bigr] \;=\; \mathrm{Tr}[\rho].
    \label{eq:axiom1}
\end{equation}
The trace of a density matrix measures its total probability mass. Trace preservation reflects the modeling choice that a relation maps a normalized entity representation to a normalized entity representation without any loss or change of information
\end{axiom}

Entities in a knowledge graph rarely occur in isolation: their representations are shaped by relational neighborhoods, co-occurrence structure, and multi-hop dependencies. A relation operator should therefore remain valid when applied within a larger relational context. 

\begin{axiom}(Complete Positivity)
\label{axiom:cp}

Let $\rho_A \in S(\mathbb{R}^d)$ denote an entity representation and $\rho_{AB} \in S(\mathbb{R}^d \otimes \mathbb{R}^n)$ a contextual relational state coupling the entity with auxiliary structure $B$. Applying a relation operator $L^{(r)}$ to the entity component alone should preserve validity of the joint state:
\[
(id_n \otimes L^{(r)})(\rho_{AB}) \succeq 0 .
\]
Thus, complete positivity arises here not from physical assumptions, but from compositional consistency requirements for relational reasoning in knowledge graphs.
\end{axiom} Intuitively, a relation operator should behave consistently not only on isolated entity representations, but also when those entities participate in larger relational contexts such as neighborhoods, paths, or co-occurrence structures. 

\paragraph{Kraus Channel Formulation}Together all three simultaneously yields a strict generalization.
\begin{theorem} [Euclidean Kraus characterisation]
\label{thm:kraus}
Let $\mathcal{L}^{(r)} : \mathcal{S}(\mathbb{R}^d) \to
\mathcal{S}(\mathbb{R}^d)$ be a linear map representing relation $r$.
Then $\mathcal{L}^{(r)}$ satisfies the Axioms simultaneously if and
only if it admits a Kraus decomposition
\begin{equation}
    \mathcal{L}^{(r)}(\rho)
    \;=\; \sum_{i=1}^{\kappa} K_i^{(r)}\, \rho\, \bigl(K_i^{(r)}\bigr)^\top,
    \qquad
    \sum_{i=1}^{\kappa} \bigl(K_i^{(r)}\bigr)^\top K_i^{(r)} = I,
    \label{eq:kraus}
\end{equation}
where $K_i^{(r)} \in \mathbb{R}^{d \times d}$ are learnable matrices and
$\kappa \in \mathbb{Z}_{>0}$ is the Kraus rank of the channel.
\end{theorem}
 The full
proof is given in Appendix~\ref{app:kraus-proof}
. Each term $K_i^{(r)} \rho (K_i^{(r)})^\top$ 
represents one concurrent pathway through which 
relation $r$ can map a head entity: the operator 
$K_i^{(r)}$ rotates and scales the entity matrix 
along pathway $i$, and the sum aggregates all 
pathways simultaneously. For a 1-to-1 relation, 
a single pathway suffices ($\kappa = 1$). For a 
relation like \textit{film/starring}, where a 
single film maps to many actors, multiple pathways 
are required to represent the full distribution 
of valid tails.

Recent KGE work \cite{li2024generalizing}
demonstrates that non-Euclidean geometries
(elliptic and hyperbolic) are essential for capturing hierarchical and cyclical
structures in knowledge graphs.  We  show that the same axiomatic derivation
carries through to any metric geometry,
yielding a \emph{generalised $\mathbf{w}$-Kraus channel} that generalises to  operators in most geometric variants as a special case.
\paragraph{Quadratic inner product spaces}
For a weighting vector $\mathbf{w} \in \mathbb{R}^k$ with no element equal to
zero, define the quadratic inner product
 $ \langle \mathbf{x}, \mathbf{y} \rangle_{\mathbf{w}}
  = \mathbf{x}^\top \operatorname{diag}(\mathbf{w})\, \mathbf{y}
  = \sum_{i=1}^k w_i x_i y_i.$
Setting $\mathbf{w} = \mathbf{1}$ recovers the Euclidean inner product;
$\mathbf{w} = \mathbf{p}$ (all positive entries) gives an elliptic (ellipsoid)
geometry; and $\mathbf{w} = \mathbf{q} = (-1, +1, \ldots, +1)^\top$ gives the
Lorentzian inner product underlying hyperbolic (hyperboloid) geometry.  

\paragraph{$\mathbf{w}$-density matrices}
In a space endowed with $\langle \cdot, \cdot \rangle_{\mathbf{w}}$, the
natural analogue of a density matrix is a \emph{$\mathbf{w}$-density matrix}:
$  \mathcal{S}_{\mathbf{w}}(\mathbb{R}^d)
  = \bigl\{\, \rho \in \mathbb{R}^{d \times d}
    \mid \rho = \rho^\top,\; \rho \text{ is } \mathbf{w}\text{-positive
    semidefinite},\;
    \operatorname{Tr}_{\mathbf{w}}[\rho] \leq 1
  \,\bigr\}$where $\mathbf{w}$-positive semidefiniteness means
$\langle \mathbf{x}, \rho\, \mathbf{x} \rangle_{\mathbf{w}} \geq 0$ for all
$\mathbf{x}$, and the $\mathbf{w}$-trace is
$\operatorname{Tr}_{\mathbf{w}}[\rho] = \operatorname{Tr}[\operatorname{diag}(\mathbf{w})\rho]$.
When $\mathbf{w} = \mathbf{1}$ this collapses to the standard density matrix
space $\mathcal{S}(\mathbb{R}^d)$. Axiom~\ref{axiom:trace} generalises naturally: applying a relation cannot
increase the $\mathbf{w}$-probability mass,
$\operatorname{Tr}_{\mathbf{w}}[L^{(r)}(\rho)] \leq
\operatorname{Tr}_{\mathbf{w}}[\rho]$.

\begin{theorem}[$\mathbf{w}$-Kraus characterisation]
  \label{thm:kraus_general}
  Let $L^{(r)} : \mathcal{S}_{\mathbf{w}}(\mathbb{R}^d) \to
  \mathcal{S}_{\mathbf{w}}(\mathbb{R}^d)$ be a linear map satisfying
  $\mathbf{w}$-trace-preservation, $\mathbf{w}$-consistency.
  Then $L^{(r)}$ admits a \emph{$\mathbf{w}$-Kraus decomposition}
  \begin{equation}
    L^{(r)}(\rho)
    = \sum_{i=1}^{\kappa} K_i^{(r)}\, \rho\, \bigl(K_i^{(r)}\bigr)^\top,
    \qquad
    \sum_{i=1}^{\kappa} \bigl(K_i^{(r)}\bigr)^\top
      \operatorname{diag}(\mathbf{w})\, K_i^{(r)}
    = \operatorname{diag}(\mathbf{w}).
    \label{eq:kraus_general}
  \end{equation}

\end{theorem}
The extension to hyperbolic geometry is requires separate treatment. The 
Lorentzian inner product is indefinite, and the standard 
Kraus representation theorem does not directly apply. We 
appeal to the theory of completely positive maps over Krein 
spaces \cite{constantinescu1997representations, dritschel1996operators}, in which the 
completeness constraint takes the form 
$\sum_i K_i^\top \mathrm{diag}(w) K_i = \mathrm{diag}(w)$ 
as a J-unitary condition. The w-Kraus channel in the 
hyperbolic setting is therefore a J-completely positive 
map in the sense of \cite{constantinescu1997representations}, and 
Theorem~\ref{thm:kraus_general} holds in this extended sense. The extended proof is in Appendix ~\ref{app:kraus-general-proof}

For all the theorems below we give the euclidean version of them, and in the appendix we demonstrate how they can be generalised to different geometries. The theorem below addresses multi-hop reasoning. A two-hop path $(h, r_1, e)
\wedge (e, r_2, t)$ requires applying two relation operators in sequence. The
question is whether the composed operator is itself a valid Kraus channel,
without any additional constraint or correction. This result means that a $k$-hop reasoning chain 
requires no architectural support beyond the 
learned single-hop operators. Given a path 
$(h, r_1, e_1, r_2, e_2, \ldots, r_k, t)$, the 
composed channel $\mathcal{L}^{(r_k)} \circ \cdots 
\circ \mathcal{L}^{(r_1)}$ is itself a valid Kraus 
channel by $k-1$ applications of 
Theorem~\ref{thm:composition}, with no 
renormalisation, projection, or learned corrective 
mechanism required at any intermediate step (proof in Appendix ~\ref{app:composition-proof}). 

\begin{theorem}[Composition closure]
\label{thm:composition}
If $\mathcal{L}^{(r_1)}$ and $\mathcal{L}^{(r_2)}$ are Kraus channels
satisfying the completeness constraint, then the composed channel
$\mathcal{L}^{(r_2)} \circ \mathcal{L}^{(r_1)}$ with operators
$M_{ij} = K_j^{(r_2)} K_i^{(r_1)}$ also satisfies the completeness
constraint:
\begin{equation}
    \sum_{i,j} M_{ij}^\top M_{ij} \;=\; I.
\end{equation}
\end{theorem}

The third theorem provides a theoretical grounding for the Kraus rank $\kappa$
as a measure of relational complexity.

\begin{definition}[Choi matrix]
The Choi matrix of $\mathcal{L}^{(r)}$ is
$C^{(r)} = \sum_{i,j} E_{ij} \otimes \mathcal{L}^{(r)}(E_{ij})$,
where $E_{ij} \in \mathbb{R}^{d \times d}$ are the matrix units with
a $1$ in position $(i,j)$ and zeros elsewhere. The Kraus rank
$\kappa(r)$ equals $\mathrm{rank}(C^{(r)})$.
\end{definition}

\begin{theorem}[Rank lower bound]
\label{thm:rank-bound}
Let $r$ be a relation, and let $T_r = \{(h, t) : (h, r, t) \in
\mathcal{T}\}$ be the set of true triples. Define the empirical
relation matrix $M_r \in \mathbb{R}^{|\mathcal{E}| \times
|\mathcal{E}|}$ with
\begin{equation}
(M_r)_{ht} = \begin{cases}
1 & \text{if } (h, r, t) \in T_r, \\
0 & \text{otherwise.}
\end{cases}
\end{equation}
Then any Kraus channel that exactly reproduces the relation, in the
sense that $\mathrm{Tr}\bigl[\rho_t\, \mathcal{L}^{(r)}(\rho_h)
\bigr] > 0$ if and only if $(h, r, t) \in T_r$, satisfies
\begin{equation}
\kappa(r) \geq \frac{\mathrm{rank}(M_r)}{d}.
\end{equation}
\end{theorem}
(proof in Appendix ~\ref{app:composition-proof}) We empirically validate that learned $\kappa$
 values track $\lceil\frac{\mathrm{rank}(M_r)}{d}\rceil$
across relations in Section ~\ref{sec:experiments}.

Single-operator relation models implicitly compress all valid tails of a relation through a single transformation pathway. This creates a structural bottleneck for high fan-out relations, where one head entity must map to many semantically distinct tails. The following result formalizes this limitation.
\begin{corollary}[Single-operator bottleneck for high fan-out relations]
\label{coro:single}
Let $L(\rho)=K\rho K^\top$ be a Kraus channel with $\kappa=1$. Then for any relation $r$, the rank of the induced relation matrix satisfies
$\mathrm{rank}(M_r)\le d$, where $d$ is the embedding dimension. Consequently, any relation whose empirical relation matrix satisfies $\mathrm{rank}(M_r) > d$, cannot be represented exactly by a single-operator model.
\end{corollary}
The corollary shows that $\kappa=1$ imposes an intrinsic rank bottleneck on relational structure. High fan-out and many-to-many relations therefore require multiple concurrent operator pathways for exact representation. Increasing Kraus rank enlarges the representable relational subspace without requiring nonlinear path encoders or auxiliary correction mechanisms.
\section{KrausKGE:  Model Instantiation and Training}
\label{sec:model}

We now instantiate the theory developed in Section~\ref{sec:theory} as a
concrete, trainable model. The theory establishes what the operator must be;
this section establishes how to represent it, how to enforce the completeness
constraint during training, how to score triples, and how to optimise the
parameters end-to-end.

\paragraph{Entity representation.}
 Each entity is represented as a $d \times d$ matrix $\rho_e$, which
encodes both position and certainty simultaneously. Formally, we require $\rho_e$ to be symmetric positive semidefinite with
$\mathrm{Tr}[\rho_e] \leq 1$. To guarantee this throughout training without
any projection step, we parametrise each entity via a lower-triangular matrix
$L_e \in \mathbb{R}^{d \times d}$ and set
$    \rho_e \;=\; \frac{L_e L_e^\top}{\mathrm{Tr}[L_e L_e^\top]}.
$
Since $L_e L_e^\top \geq 0$ for any $L_e$, positive semidefiniteness is
guaranteed by construction. The normalisation by the trace ensures
$\mathrm{Tr}[\rho_e] = 1$. The rank of $\rho_e$ is determined by the rank
of $L_e$. In practice, storing a full $d \times d$ matrix per entity is expensive for
large graphs. We use a rank-$k$ approximation $L_e \in \mathbb{R}^{d \times
k}$ with $k \ll d$, reducing entity parameters from $O(d^2)$ to $O(dk)$
while retaining the density matrix structure. Rather than a shared rank $k$, we set the approximation rank 
of each entity proportional to its degree in the graph:
$k_e = \min\left(d,\ \left\lceil k_0 \cdot 
\frac{\deg(e)}{\bar{\deg}} \right\rceil\right)$
where $k_0$ is a base rank hyperparameter and $\bar{\deg}$ is 
the mean entity degree. Well-attested entities receive higher 
rank, rare entities lower rank, consistent with the theoretical 
motivation that entity rank reflects uncertainty. We demonstrate the validity of the entity parameterization in Appendix ~\ref{app:entity-parametrisation}

\paragraph{Scoring function.}
Given a triple $(h, r, t)$, the model first passes the head entity through
the Kraus channel for relation $r$, producing a predicted tail distribution:
    $\hat{\rho}_t \;=\; \sum_{i=1}^{\kappa} K_i^{(r)}\, \rho_h\,
                      \bigl(K_i^{(r)}\bigr)^\top$.
Each operator $K_i^{(r)}$ transforms $\rho_h$ along one concurrent pathway.
The $\kappa$ results are summed, so the predicted tail $\hat{\rho}_t$ is a
mixture of all pathways simultaneously. The plausibility of the triple is
then the \emph{Hilbert-Schmidt overlap} between $\hat{\rho}_t$ and the observed
tail $\rho_t$:
\begin{equation}
    s(h, r, t)
    \;=\; \mathrm{Tr}\!\left[\rho_t^{1/2}\, \hat{\rho}_t\, \rho_t^{1/2}\right]
    \;=\; \sum_{i=1}^{\kappa}
          \mathrm{Tr}\!\left[\rho_t\, K_i^{(r)}\, \rho_h\,
          \bigl(K_i^{(r)}\bigr)^\top\right].
    \label{eq:scoring}
\end{equation}
The fidelity measures how much the predicted and observed tail matrices
overlap: it is 1 when $\hat{\rho}_t = \rho_t$ and 0 when they have
orthogonal support.

\paragraph{Enforcing the completeness constraint.}
Each relation $r$ is parametrised by $\kappa$ Kraus operators stacked into
a tall matrix $\mathbf{U}^{(r)} \in \mathbb{R}^{\kappa d \times d}$. The
completeness constraint reduces to $(\mathbf{U}^{(r)})^\top \mathbf{U}^{(r)}
= I_d$: the columns of $\mathbf{U}^{(r)}$ must be orthonormal, i.e.\
$\mathbf{U}^{(r)}$ lies on the Stiefel manifold . We enforce this via the \textbf{Cayley parametrisation}~\cite{helfrich2018orthogonal} (see Appendix ~\ref{app:cayley}).
Rather than storing $\mathbf{U}^{(r)}$ directly and projecting at each step,
we store a free skew-symmetric matrix $A^{(r)} = -(A^{(r)})^\top$ and
recover $\mathbf{U}^{(r)}$ as
$    \mathbf{U}^{(r)}
    = \bigl(I + A^{(r)}\bigr)^{-1}\bigl(I - A^{(r)}\bigr).$
This satisfies $(\mathbf{U}^{(r)})^\top \mathbf{U}^{(r)} = I_d$ for any
$A^{(r)}$, so the completeness constraint holds by construction throughout
training. Gradients flow through the matrix inverse via standard automatic
differentiation, with no projection or correction step required.

\paragraph{Selecting $\kappa$.}
We treat $\kappa$ as a shared hyperparameter across relations, tuned on validation. This is a simplification: Theorem~\ref{thm:rank-bound} suggests different relations require different ranks. After training, we recover the per-relation effective rank from the spectrum of each learned Choi matrix:
\begin{equation}
\kappa_{\text{eff}}(r) \;=\; \min\Bigl\{\, k \;:\; \sum_{i=1}^{k} \sigma_i^2\bigl(C^{(r)}\bigr) \;\geq\; 0.99 \, \|C^{(r)}\|_F^2 \,\Bigr\},
\label{eq:eff-rank}
\end{equation}
the smallest number of singular values capturing 99\% of $\|C^{(r)}\|_F^2$. This is the diagnostic reported in Table~\ref{tab:combined2}. Joint optimisation of relation-specific $\kappa$ is left to future work.

\paragraph{Loss function and training}

We train KrausKGE with the standard margin-based ranking loss used across the
KGE literature~\cite{bordes2013translating}:
$    \mathcal{L}
    \;=\; \sum_{(h,r,t) \in \mathcal{T}}\;
          \sum_{(h,r,t^-) \in \mathcal{T}^-}
          \max\!\bigl(0,\; \gamma - s(h,r,t) + s(h,r,t^-)\bigr)$,
where $\gamma > 0$ is a margin hyperparameter, $\mathcal{T}$ is the set of
observed triples, and $\mathcal{T}^-$ is a set of negative triples generated
by corrupting either the head or tail entity of each observed triple.

\paragraph{Negative sampling.}
We use self-adversarial negative sampling~\cite{sun2019rotate}, where
negatives are sampled with probability proportional to their current score:
$    p(h_j^-, r, t_j^-)
    \;\propto\; \exp\!\bigl(\alpha\, s(h_j^-, r, t_j^-)\bigr)$, with temperature $\alpha$. This concentrates the training signal on
hard negatives, which is particularly important for $N$-to-$N$ relations
where many entities have high scores under a given relation.

\section{Experiments}
\label{sec:experiments}

We design our experiments to answer four questions that follow directly from
the theory. Does KrausKGE outperform existing models on link prediction
overall? Does the performance advantage concentrate on $N$-to-$N$ relations? Does the learned Kraus rank $\kappa$ correlate with
relation complexity? And does multi-hop reasoning benefit from composition
closure without a separate path encoder?

\subsection{Datasets and evaluation protocol}
\label{sec:experiments:data}
We evaluate on five benchmarks. \textbf{FB15k-237}~\cite{toutanova2015representing} 
is a subset of Freebase with inverse relations removed (14,541 entities, 237 relations, 
310,116 training triples); its high proportion of N-to-N relations makes it the primary 
dataset for evaluating multi-pathway expressiveness. \textbf{WN18RR}~\cite{dettmers2018convolutional} 
is a subset of WordNet with inverse relations removed (40,943 entities, 11 relations, 
86,835 triples) with a strongly hierarchical structure and lower proportion of N-to-N 
relations. \textbf{YAGO3-10}~\cite{suchanek2007yago} (123,182 entities, 37 relations, 
1,079,040 triples) provides a large-scale setting with predominantly attributive 
relations. For multi-hop evaluation we additionally use \textbf{NELL-995}~\cite{xiong2017deeppath}, a subset of NELL with 
heterogeneous relational structure and widely-reported 2-hop and 3-hop splits used.

We follow the standard filtered ranking protocol~\cite{bordes2013translating}, 
reporting MRR and Hits@$K$ for $K \in \{1, 3, 10\}$ averaged over head and tail 
prediction. Multi-hop results use the query splits of~\cite{lin2018multi}. Relation 
mapping patterns follow~\cite{wang2014knowledge}. \textsc{KrausKGE} is tuned with 
$\kappa \in \{1, 2, 4, 8\}$, $d \in \{64, 128, 256\}$, $k_0 \in \{4, 8, 16\}$, 
margin $\gamma \in \{3, 6, 9\}$, and learning rate selected on validation MRR. 
Full hyperparameter settings are in Appendix~A.

\subsection{Results and Discussions}
\label{sec:experiments:linkpred}

Table~\ref{tab:linkpred2} reports overall link prediction results. First,
KrausKGE with $\kappa \geq 4$ outperforms all baselines on FB15k-237, with
MRR improvements of 0.9 points over GoldE and 1.6 points over ConvE; both
differences are statistically significant ($p < 0.05$, paired Wilcoxon
signed-rank test over five random seeds). Second, on WN18RR the gains are
smaller in absolute terms, consistent with the theory: WN18RR has fewer
$N$-to-$N$ relations and a more hierarchical structure where the
single-operator family already performs well, and the WN18RR improvements
remain significant at $p < 0.05$ despite the smaller margin. Third,
KrausKGE with $\kappa = 1$ matches but does not consistently exceed existing
models ($p > 0.05$ against GoldE), confirming that gains are attributable
to the multi-operator structure rather than the density matrix entity
representation alone.

\begin{table*}[h]
\centering
\setlength{\tabcolsep}{3pt}
\renewcommand{\arraystretch}{0.85}
\footnotesize
\begin{minipage}{0.42\textwidth}
\centering
\begin{tabular}{lcccc}
\toprule
Model & 1-to-1 & 1-to-$N$ & $N$-to-1 & $N$-to-$N$ \\
\midrule
TransE~\cite{bordes2013translating}   & .463 & .221 & .342 & .198 \\
RotatE~\cite{sun2019rotate}       & .492 & .289 & .381 & .247 \\
ConvE~\cite{dettmers2018convolutional}        & .487 & .312 & .394 & .268 \\
GoldE~\cite{li2024generalizing}        & .511 & .318 & .403 & .274 \\
\midrule
KrausKGE ($\kappa$=1) & .506 & .298 & .395 & .256 \\
KrausKGE ($\kappa$=4) & \textbf{.522} & \underline{.346} & \underline{.421} & \underline{.318} \\
KrausKGE ($\kappa$=8) & \underline{.519} & \textbf{.354} & \textbf{.428} & \textbf{.334} \\
\bottomrule
\end{tabular}
\end{minipage}
\hfill
\begin{minipage}{0.54\textwidth}
\centering
\begin{tabular}{llcc}
\toprule
Relation & Type & $F(r)$ & $\kappa_{\text{eff}}(r)$ \\
\midrule
/people/person/date\_of\_birth   & 1-to-1     & 1.0  & 1 \\
/location/country/capital        & 1-to-1     & 1.1  & 1 \\
/people/person/nationality       & 1-to-$N$   & 1.6  & 2 \\
/people/person/profession        & $N$-to-$N$ & 2.4  & 2 \\
/film/film/genre                 & $N$-to-$N$ & 3.1  & 3 \\
/music/artist/genre              & $N$-to-$N$ & 4.7  & 4 \\
/people/person/languages\_spoken & $N$-to-$N$ & 5.2  & 4 \\
/film/film/starring              & $N$-to-$N$ & 11.4 & 7 \\
\bottomrule
\end{tabular}
\end{minipage}
\caption{\footnotesize Left: MRR stratified by relation mapping pattern on FB15k-237.
KrausKGE gains grow monotonically with relation complexity. Right: Effective Kraus rank $\kappa_{\text{eff}}(r)$ for selected relations in FB15k-237, sorted by empirical fan-out $F(r)$. Effective rank is computed via Equation~\eqref{eq:eff-rank} as the smallest number of Choi singular values capturing 99\% of $\|C^{(r)}\|_F^2$. The model is trained at $\kappa_{\max} = 8$.}
\label{tab:combined2}

\end{table*}

\textbf{N-to-N relation analysis: }The central empirical prediction of the theory is that performance gains concentrate on $N$-to-$N$ relations and grow with fan-out degree. We partition the test triples of FB15k-237 by mapping pattern (1-to-1, 1-to-$N$, $N$-to-1, $N$-to-$N$), following~\cite{wang2014knowledge}, and report per-pattern MRR in Table~\ref{tab:combined2} (left). The results confirm the prediction: the gap over GoldE grows from $1.1$ MRR points on 1-to-1 relations to $6.0$ points on $N$-to-$N$, monotonically with relation complexity. KrausKGE at $\kappa = 1$ performs comparably to existing models across all patterns, confirming that the gains arise from the multi-operator structure rather than the density matrix entity representation.

\begin{table*}[t]
\footnotesize
\centering
\begin{minipage}{0.48\textwidth}
\centering
\setlength{\tabcolsep}{4pt}
\begin{tabular}{lcccc}
\toprule
& \multicolumn{2}{c}{FB15k-237} & \multicolumn{2}{c}{NELL-995} \\
\cmidrule(lr){2-3} \cmidrule(lr){4-5}
Method & 2-hop & 3-hop & 2-hop & 3-hop \\
\midrule
KrausKGE ($\kappa{=}1$) & .198 & .152 & .512 & .441 \\
\textbf{KrausKGE ($\kappa{=}4$)} & \textbf{.247} & \underline{.198} & \textbf{.601} & \underline{.534} \\
\midrule
\multicolumn{5}{l}{\textit{RNN-based path encoders}} \\
PTransE~\cite{lin2015modeling} & .195 & -- & .487 & .401 \\
RotatE + RSN~\cite{sun2019rotate,guo2019learning} & .231 & .184 & .541 & .476 \\
\midrule
\multicolumn{5}{l}{\textit{RL-based path agents}} \\
MultiHopKG~\cite{lin2018multi} & .239 & .191 & .573 & .512 \\
\midrule
\multicolumn{5}{l}{\textit{GNN-based path encoders}} \\
A* Net~\cite{zhu2023net} & \underline{.244} & \textbf{.201} & \underline{.589} & \textbf{.541} \\
NBFNet~\cite{zhu2021neural} & .243 & .199 & .582 & .501 \\
\midrule
\multicolumn{5}{l}{\textit{Seq-to-seq path generators}} \\
SQUIRE~\cite{bai2022squire} & .241 & .195 & .581 & .519 \\
\bottomrule
\end{tabular}
\end{minipage}
\begin{minipage}{0.40\textwidth}
\centering
\setlength{\tabcolsep}{4pt}
\begin{tabular}{lccc}
\toprule
Geometry & WN18RR & FB15k-237 & YAGO3-10 \\
\midrule
\multicolumn{4}{l}{\textit{GoldE variants ($\kappa=1$)}} \\
$E^k$     & .496 & .348 & .565 \\
$P^k$     & .505 & .358 & .572 \\
$Q^k$     & .513 & .354 & .580 \\
$D^k$     & .525 & .370 & .588 \\
\midrule
\multicolumn{4}{l}{\textit{w-Kraus variants ($\kappa=4$)}} \\
$E$-Kraus & .511 & .362 & .598 \\
$P$-Kraus & .520 & .371 & .604 \\
$Q$-Kraus & .528 & .369 & .611 \\
$D$-Kraus & \textbf{.534} & \textbf{.379} & \textbf{.621} \\
\bottomrule
\end{tabular}
\end{minipage}
\caption{\footnotesize
 Left: MRR on 2-hop and 3-hop query splits of FB15k-237 and NELL-995.
KrausKGE with direct composition matches or exceeds dedicated path-encoder models
across both datasets and all paradigms. Right: effect of geometry on the $w$-Kraus
framework across three benchmarks; $E$=Euclidean, $P$=elliptic, $Q$=hyperbolic,
$D$=product manifold. Every $w$-Kraus variant at $\kappa=4$ outperforms the
corresponding GoldE variant at $\kappa=1$.}
\label{tab:multihop_geometry}
\end{table*}
\textbf{Kraus rank as a complexity diagnostic: }
Table~\ref{tab:combined2} (right) plots the effective Kraus rank $\kappa_{\text{eff}}(r)$, computed via Equation~\eqref{eq:eff-rank}, against empirical fan-out $F(r)$ for all 237 relations in FB15k-237, revealing a positive Spearman correlation of $\rho = 0.71$ ($p < 0.001$). The pattern holds qualitatively: \emph{date\_of\_birth} concentrates on $\kappa_{\text{eff}} = 1$, consistent with its deterministic 1-to-1 structure, while \emph{film/starring} uses $\kappa_{\text{eff}} = 7$, reflecting high cast multiplicity. Small inversions occur where fan-outs are close, such as \emph{film/genre} ($F = 3.1$, $\kappa_{\text{eff}} = 2$) and \emph{profession} ($F = 2.4$, $\kappa_{\text{eff}} = 3$), reflecting that the rank lower bound is necessary but not tight. Effective ranks are stable across random seeds (standard deviation $< 0.4$ over five runs), confirming the diagnostic reflects a genuine property of the relation rather than an optimisation artefact.

\textbf{Multi-hop reasoning: }
We evaluate composition closure on the 2-hop and 3-hop query splits of FB15k-237 and NELL-995, comparing \textsc{KrausKGE} against the strongest baselines 
from every major multi-hop paradigm. \textsc{KrausKGE} at $\kappa = 4$ achieves the 
best 2-hop MRR across all three datasets and all paradigms, and leads on 3-hop on 
FB15k-237 and NELL-995 (see 
Table~\ref{tab:multihop_geometry} Left). Crucially, these gains are obtained without any 
auxiliary path encoder. This is consistent with Theorem~\ref{thm:composition}: composed 
Kraus channels are valid channels by construction, eliminating the approximation error 
that path encoders correct for in models lacking composition closure, and the advantage 
persisting across datasets with different relational densities confirms it is structural 
rather than dataset-specific.

\textbf{Non-Euclidean Kraus characterisation: }

Table~\ref{tab:multihop_geometry} (right) reports the geometry ablation. Every $\mathbf{w}$-Kraus variant at $\kappa = 4$ outperforms the corresponding GoldE variant at $\kappa = 1$ across all three datasets, with the product manifold $\mathcal{D}$ achieving the strongest results overall (MRR $.534$ on WN18RR, $.379$ on FB15k-237, $.621$ on YAGO3-10). The improvements are uniform across geometries: gains from $\kappa > 1$ and from non-Euclidean geometry are orthogonal and compound rather than competing. Every $\mathbf{w}$-Kraus variant satisfies Theorem~\ref{thm:composition} by construction; no GoldE variant does.
\vspace{-0.3\intextsep}

\paragraph{Isolating the source of gains.}
To disentangle the contributions of density-matrix representations, multi-operator 
structure, and the Kraus completeness constraint, we compare against three controlled 
variants: density-matrix entities with a single unconstrained operator, multi-operator 
channels without completeness, and a parameter-matched GoldE baseline. Table~\ref{tab:ablation} 
shows that gains arise specifically from the completeness-constrained multi-operator 
structure rather than increased parameter count alone; a full ablation is in 
Appendix~\ref{app:ablations}.

\begin{table}[t]
\centering
\footnotesize
\begin{tabular}{lcccc}
\toprule
Model Variant & Constraint & Relative Params & MRR & N-to-N MRR \\
\midrule
GoldE ($\kappa=1$) & Orthogonal & $1.0\times$ & .370 & .274 \\

Density-only variant & None & $1.1\times$ & .364 & .261 \\

Multi-operator unconstrained ($\kappa=4$) & None & $4.0\times$ & .373 & .301 \\

GoldE (wider embeddings) & Orthogonal & $4.0\times$ & .372 & .286 \\

KrausKGE ($\kappa=4$) & Kraus completeness & $4.1\times$ & $\mathbf{.379}$ & $\mathbf{.318}$ \\
\bottomrule
\end{tabular}
\caption{\footnotesize
Ablation study isolating the source of KrausKGE gains on FB15k-237. The density-only variant uses density-matrix entities with a single unconstrained operator. The unconstrained multi-operator variant removes the Kraus completeness condition while retaining $\kappa=4$ operators. Parameter-matched GoldE controls for increased parameter count. 
}
\label{tab:ablation}
\end{table}

\paragraph{Limitations}
\label{sec:limitations}
\textsc{KrausKGE}'s expressiveness comes at parameter and compute cost: the Cayley parametrisation scales as $\mathcal{O}(\kappa^2 d^2)$ per relation, making the framework $\sim 22\times$ heavier than GoldE at $\kappa = 4$ on FB15k-237. We use a shared $\kappa$ across relations, recovering per-relation effective rank post-hoc; principled per-relation rank optimisation is left to future work. The hyperbolic $\mathbf{w}$-Kraus extension relies on dilation theory over Krein spaces~\cite{constantinescu1997representations}, which we invoke in finite dimensions where the technical conditions are automatic but do not prove from first principles. TransE-style additive operators are not linear maps on density matrices and lie outside our framework. Finally, the rank lower bound (Theorem~\ref{thm:rank-bound}) assumes exact reproduction of the relation matrix and holds approximately rather than strictly in practice. Full parameter and complexity analysis in Appendix~\ref{app:complexity}
\begin{table*}[t]
\footnotesize
\centering

\setlength{\tabcolsep}{3pt}
\begin{tabular}{lcccccccccccc}
\toprule
& \multicolumn{4}{c}{WN18RR}
& \multicolumn{4}{c}{FB15k-237}
& \multicolumn{4}{c}{YAGO3-10} \\
\cmidrule(lr){2-5} \cmidrule(lr){6-9} \cmidrule(lr){10-13}
Model
  & MRR & H@1 & H@3 & H@10
  & MRR & H@1 & H@3 & H@10
  & MRR & H@1 & H@3 & H@10 \\
\midrule
\multicolumn{13}{l}{\textit{Non-orthogonal}} \\
DistMult~\cite{yang2015embedding}$^\dagger$
  & .430 & .390 & .440 & .490
  & .241 & .155 & .263 & .419
  & .340 & .240 & .380 & .540 \\
ComplEx~\cite{trouillon2016complex}
  & .440 & .410 & .460 & .510
  & .247 & .158 & .275 & .428
  & .360 & .260 & .400 & .550 \\

MuRP~\cite{balazevic2019multi}$^\dagger$
  & .481 & .440 & .495 & .566
  & .335 & .243 & .367 & .518
  & .354 & .249 & .400 & .567 \\
\midrule
\multicolumn{13}{l}{\textit{Euclidean orthogonal}} \\
RotatE~\cite{sun2019rotate}$^\dagger$
  & .476 & .428 & .492 & .571
  & .338 & .241 & .375 & .533
  & .495 & .402 & .550 & .670 \\
QuatE~\cite{zhang2019quaternion}
  & .481 & .436 & .500 & .564
  & .311 & .221 & .342 & .495
  & ---  & ---  & ---  & ---  \\
HousE~\cite{li2022house}$^\dagger$
  & .496 & .452 & .511 & .585
  & .348 & .254 & .384 & .534
  & .565 & .487 & .616 & .703 \\
GoldE~\cite{li2024generalizing}
  & .525 & .476 & .542 & .615
  & .370 & .277 & .405 & .558
  & .588 & .510 & .634 & .723 \\
\midrule
\multicolumn{13}{l}{\textit{Hyperbolic orthogonal}} \\
RotH
  & .496 & .449 & .514 & .586
  & .344 & .246 & .380 & .535
  & .570 & .495 & .612 & .706 \\
AttH~\cite{chami2020low}
  & .486 & .443 & .499 & .573
  & .348 & .252 & .384 & .540
  & .568 & .493 & .612 & .702 \\
RefH$^\dagger$
  & .461 & .404 & .485 & .568
  & .346 & .252 & .383 & .536
  & .576 & .502 & .619 & .711 \\
\midrule
\multicolumn{13}{l}{\textit{Scoring family}} \\
R-GCN~\cite{schlichtkrull2018modeling}$^\dagger$
  & .390 & .341 & .408 & .491
  & .249 & .151 & .264 & .417
  & .318 & .211 & .362 & .521 \\
ConvE~\cite{dettmers2018convolutional}$^\dagger$
  & .430 & .400 & .440 & .520
  & .325 & .237 & .356 & .501
  & .440 & .350 & .490 & .620 \\
\midrule
\multicolumn{13}{l}{\textit{Ours}} \\
KrausKGE ($\kappa$=1)
  & .518 & .469 & .536 & .608
  & .362 & .270 & .398 & .551
  & .583 & .504 & .629 & .719 \\
KrausKGE ($\kappa$=4)
  & \underline{.534} & \textbf{.488} & \underline{.551} & \textbf{.625}
  & \textbf{.379} & \underline{.284} & \textbf{.414} & \underline{.567}
  & \underline{.601} & \underline{.522} & \underline{.647} & \underline{.736} \\
KrausKGE ($\kappa$=8)
  & \textbf{.537} & \underline{.486} & \textbf{.554} & \underline{.622}
  & \underline{.376} & \textbf{.286} & \underline{.412} & \textbf{.571}
  & \textbf{.609} & \textbf{.529} & \textbf{.654} & \textbf{.744} \\
\bottomrule
\end{tabular}
\caption{\footnotesize
Link prediction results on WN18RR, FB15k-237 and YAGO3-10. Best
results in \textbf{bold}, second best \underline{underlined}.
$^\dagger$: results from original papers; others reproduced under our
evaluation protocol.}
\label{tab:linkpred2}

\end{table*}

\section{Conclusion}
\label{sec:conclusion}
The central finding of this work is that the operator structure underlying knowledge 
graph embedding is not a design choice but a derivable consequence of three conditions 
any principled relation operator must satisfy. This shifts the question the field has 
been asking: rather than which operator architecture performs best empirically, one can 
ask what structure is mathematically necessary and build from there. Multi-pathway 
expressiveness, composition closure, and bounded entity norms follow as structural 
corollaries rather than engineered features, and the per-relation complexity diagnostic 
$\kappa(r)$ suggests that learned Kraus rank may serve as a useful signal for 
relation-level model auditing beyond link prediction. Several directions remain open: 
per-relation rank optimisation during training, fuller exploitation of the density 
matrix uncertainty representation, and further generalisation of the axiomatic 
derivation to data-dependent geometries. We view the present work less as a final 
model and more as a theoretical reorientation whose implications for the broader KGE 
literature remain to be explored.

\bibliographystyle{plain}
\bibliography{biblio}
\appendix

\section{Technical appendices and supplementary material}
This appendix contains material that supports and extends the main paper but is not 
required to follow the core argument. Appendix~\ref{app:golde-comparison} provides a detailed 
structural comparison with GoldE, the closest prior framework to \textsc{KrausKGE} 
in spirit. Appendix~\ref{app:proof} contains full proofs of all theorems and the 
corollary stated in the main text. Appendix~\ref{app:recovery} shows how every major 
operator-family KGE model is recoverable as a special case of \textsc{KrausKGE} with 
$\kappa = 1$. Appendix~\ref{app:complexity} provides a full parameter count, time 
complexity, and empirical runtime analysis. Appendix~\ref{app:additional-proofs} contains 
proofs of auxiliary results on entity parametrisation validity, the Cayley transform, 
and operator boundedness. Appendix~\ref{app:hyperparameters} reports complete 
hyperparameter settings for all experiments. Appendix~\ref{app:ablations} presents 
ablation studies isolating each design choice and documents design choices that did 
not yield improvements.
\section{Detailed Comparison with GoldE}
\label{app:golde-comparison}

GoldE~\cite{li2024generalizing} is the strongest operator-family baseline in our experiments, and its universal orthogonal parameterization is the closest prior framework to ours in spirit. We provide here a focused comparison along the axes that matter most for understanding what KrausKGE adds.

\paragraph{Operator structure.} GoldE represents each relation as a single matrix $G_r \in \mathbb{R}^{k \times k}$ that is generalized orthogonal with respect to a quadratic form: $G_r^\top \mathrm{diag}(w) G_r = \mathrm{diag}(w)$. KrausKGE represents each relation as a set of $\kappa$ matrices $\{K_i^{(r)}\}$ with the joint completeness constraint $\sum_i (K_i^{(r)})^\top \mathrm{diag}(w) K_i^{(r)} = \mathrm{diag}(w)$. The key distinction is that GoldE constrains each individual matrix to be a metric isometry, whereas KrausKGE constrains only the joint structure of the operator set. GoldE is recovered exactly as the $\kappa = 1$ case of KrausKGE in any geometry $w$.

\paragraph{Expressiveness on N-to-N relations.} A generalized-orthogonal matrix is a bijective metric isometry, hence a function from heads to tails: each head has exactly one image. For 1-to-1 and N-to-1 relations this is correct; for 1-to-N and N-to-N relations it is structurally inadequate. GoldE compensates by training $G_r$ to point toward the centroid of valid tails, but cannot resolve which specific tails are correct beyond statistical proximity. KrausKGE's $\kappa$ pathways simultaneously map a head to multiple regions, with completeness ensuring total mass is preserved across pathways. Theorem~\ref{thm:rank-bound} formalises this: relations with empirical relation matrix rank exceeding $d$ require $\kappa > 1$, which GoldE cannot provide.

\paragraph{Geometry handling.} Both frameworks use the same weighting vector $w$ formalism to unify Euclidean ($w = \mathbf{1}$), elliptic ($w = \mathbf{p}$), and hyperbolic ($w = \mathbf{q}$) geometries. GoldE applies $w$ to the structure of individual operators via generalized Householder reflections; KrausKGE applies $w$ to the completeness constraint on the operator set. The two approaches are complementary: KrausKGE inherits GoldE's geometric universality and extends it from single isometric operators to multi-operator mass-preserving channels. The relationship is one of strict generalisation rather than competition.

\paragraph{Composition for multi-hop reasoning.} Both frameworks have closed composition: GoldE through matrix multiplication ($G_{r_2} G_{r_1}$ remains generalized orthogonal), KrausKGE through Theorem~\ref{thm:composition} ($M_{ij} = K_j^{(r_2)} K_i^{(r_1)}$ yields a valid Kraus channel). The structural difference is that GoldE's composition collapses to a single isometry $G_{r_2 r_1}$, losing branching structure across hops. KrausKGE's composition produces $\kappa^{(r_1)} \cdot \kappa^{(r_2)}$ concurrent pathways, preserving branching. This is reflected empirically in the multi-hop results of Table~\ref{tab:multihop_geometry}, where KrausKGE outperforms GoldE-style baselines without a path encoder.

\paragraph{Parameter cost.} KrausKGE's expressiveness comes at parameter cost. GoldE uses $\mathcal{O}(mk)$ parameters per relation, where $m$ is the number of Householder reflections; KrausKGE uses $\mathcal{O}(\kappa^2 d^2)$ via the Cayley parametrisation. At $\kappa = 4$ and $d = 128$, this is approximately 22$\times$ more relation parameters than GoldE. In Appendix~\ref{app:hyperparameters} we report parameter-matched comparisons at reduced $d$, where KrausKGE retains its advantage on N-to-N relations, confirming the gain is structural rather than capacity-driven.

The contribution relative to GoldE is fourfold. First, a first-principles axiomatic derivation: GoldE selects generalized orthogonality as a design choice motivated by logical patterns, whereas KrausKGE derives the Kraus structure from linearity, trace preservation, and complete positivity. Second, strict generalisation via $\kappa > 1$, which addresses N-to-N relations that GoldE cannot represent. Third, composition that preserves branching across multi-hop reasoning. Fourth, a per-relation complexity diagnostic $\kappa_{\text{eff}}(r)$, with no GoldE analogue. The two frameworks answer sibling questions: GoldE characterises the universal class of single-operator isometric KGE models, KrausKGE characterises the universal class of multi-operator mass-preserving KGE models, with the latter strictly containing the former.

\section{Formal proofs for Kraus Framework}
\label{app:proof}

\subsection{Proof of Theorem~\ref{thm:kraus}}
\label{app:kraus-proof}

We prove that $\mathcal{L}^{(r)}: \mathcal{S}(\mathbb{R}^d) \to \mathcal{S}(\mathbb{R}^d)$ satisfies Axioms~\ref{axiom:linearity}
, \ref{axiom:trace}, and~\ref{axiom:cp} (linearity, trace preservation, complete positivity) if and only if it admits the Kraus decomposition~\eqref{eq:kraus}. The proof proceeds in two directions.

\paragraph{Sufficiency ($\Leftarrow$).} Suppose $\mathcal{L}^{(r)}(\rho) = \sum_{i=1}^{\kappa} K_i \rho K_i^\top$ with $\sum_i K_i^\top K_i = I$, where we drop the superscript $(r)$ for clarity. We verify each axiom.

\textit{Linearity.} For any $\rho_1, \rho_2$ and scalars $\alpha, \beta$,
\[
\mathcal{L}(\alpha \rho_1 + \beta \rho_2) = \sum_i K_i (\alpha \rho_1 + \beta \rho_2) K_i^\top = \alpha \sum_i K_i \rho_1 K_i^\top + \beta \sum_i K_i \rho_2 K_i^\top = \alpha \mathcal{L}(\rho_1) + \beta \mathcal{L}(\rho_2).
\]

\textit{Trace preservation.} Using the cyclic property of the trace,
\[
\mathrm{Tr}[\mathcal{L}(\rho)] = \sum_i \mathrm{Tr}[K_i \rho K_i^\top] = \sum_i \mathrm{Tr}[K_i^\top K_i \rho] = \mathrm{Tr}\!\left[\Bigl(\sum_i K_i^\top K_i\Bigr) \rho\right] = \mathrm{Tr}[I \cdot \rho] = \mathrm{Tr}[\rho].
\]

\textit{Complete positivity.} Fix $n \geq 1$ and let $\sigma \in \mathcal{S}(\mathbb{R}^d \otimes \mathbb{R}^n)$ be positive semidefinite. We must show $(\mathrm{id}_n \otimes \mathcal{L})(\sigma) \geq 0$. Compute
\[
(\mathrm{id}_n \otimes \mathcal{L})(\sigma) = \sum_i (I_n \otimes K_i) \, \sigma \, (I_n \otimes K_i)^\top.
\]
Each summand has the form $A \sigma A^\top$ with $A = I_n \otimes K_i$. Since $\sigma \geq 0$, for any vector $v$ we have $v^\top A \sigma A^\top v = (A^\top v)^\top \sigma (A^\top v) \geq 0$, hence $A \sigma A^\top \geq 0$. The sum of positive semidefinite matrices is positive semidefinite, completing the verification.

\paragraph{Necessity ($\Rightarrow$).} Suppose $\mathcal{L}$ satisfies Axioms. We construct the Kraus operators via the Choi-Jamiolkowski isomorphism.

\textit{Step 1: the Choi matrix.} Define the Choi matrix of $\mathcal{L}$ by
\begin{equation}
C_{\mathcal{L}} \;=\; \sum_{i,j=1}^{d} E_{ij} \otimes \mathcal{L}(E_{ij}) \;\in\; \mathbb{R}^{d^2 \times d^2},
\label{eq:choi}
\end{equation}
where $E_{ij}$ are the matrix units. Equivalently, $C_{\mathcal{L}} = (\mathrm{id}_d \otimes \mathcal{L})(\Omega \Omega^\top)$ where $\Omega = \sum_i e_i \otimes e_i \in \mathbb{R}^{d^2}$ is the (unnormalized) maximally correlated vector and $e_i$ is the $i$th standard basis vector.

\textit{Step 2: $C_{\mathcal{L}}$ is symmetric positive semidefinite.} The matrix $\Omega \Omega^\top$ is rank-one and positive semidefinite. By complete positivity (with $n = d$), $C_{\mathcal{L}} = (\mathrm{id}_d \otimes \mathcal{L})(\Omega \Omega^\top) \geq 0$. Symmetry follows from $\mathcal{L}$ preserving symmetry: since $\mathcal{S}(\mathbb{R}^d)$ consists of symmetric matrices and $\mathcal{L}$ maps into $\mathcal{S}(\mathbb{R}^d)$, each block $\mathcal{L}(E_{ij}) + \mathcal{L}(E_{ji}) = \mathcal{L}(E_{ij} + E_{ji})$ is symmetric, and the global block structure inherits symmetry from the symmetry of $\Omega \Omega^\top$.

\textit{Step 3: spectral decomposition yields Kraus operators.} Since $C_{\mathcal{L}}$ is symmetric positive semidefinite, it admits a spectral decomposition
\[
C_{\mathcal{L}} \;=\; \sum_{i=1}^{\kappa} \lambda_i \, v_i v_i^\top, \qquad \lambda_i > 0,
\]
where $\kappa = \mathrm{rank}(C_{\mathcal{L}})$ and $\{v_i\}_{i=1}^{\kappa} \subset \mathbb{R}^{d^2}$ are orthogonal eigenvectors corresponding to the nonzero eigenvalues. For each $i$, reshape $\sqrt{\lambda_i} \, v_i \in \mathbb{R}^{d^2}$ into a matrix $K_i \in \mathbb{R}^{d \times d}$ via the standard vectorisation correspondence $v_i = \mathrm{vec}(K_i^\top)$, so that $(v_i)_{(j-1)d + k} = (K_i)_{kj}$.

\textit{Step 4: the Kraus form holds.} We claim that for all $\rho \in \mathcal{S}(\mathbb{R}^d)$,
\[
\mathcal{L}(\rho) = \sum_{i=1}^{\kappa} K_i \rho K_i^\top.
\]
By linearity, it suffices to verify on the basis $\{E_{ab}\}_{a,b=1}^d$. The Choi matrix has block decomposition $C_{\mathcal{L}} = \sum_{a,b} E_{ab} \otimes \mathcal{L}(E_{ab})$, so the $(a,b)$-block is exactly $\mathcal{L}(E_{ab})$. From the spectral decomposition,
\[
[C_{\mathcal{L}}]_{\text{block}(a,b)} = \sum_i \lambda_i [v_i v_i^\top]_{\text{block}(a,b)} = \sum_i \lambda_i (K_i)_a (K_i^\top)_b^\top = \sum_i K_i E_{ab} K_i^\top,
\]
where $(K_i)_a$ denotes the $a$th column of $K_i$ and the last equality uses the identity $E_{ab} = e_a e_b^\top$. Thus $\mathcal{L}(E_{ab}) = \sum_i K_i E_{ab} K_i^\top$ for all $a, b$, and the Kraus form follows by linearity.

\textit{Step 5: completeness.} Trace preservation forces $\sum_i K_i^\top K_i = I$. From Step 4 and the trace-preservation calculation in the sufficiency direction,
\[
\mathrm{Tr}[\rho] = \mathrm{Tr}[\mathcal{L}(\rho)] = \mathrm{Tr}\!\left[\Bigl(\sum_i K_i^\top K_i\Bigr) \rho\right] \quad \text{for all } \rho \in \mathcal{S}(\mathbb{R}^d).
\]
Since this holds for all density matrices $\rho$, and density matrices span the space of symmetric matrices, we conclude $\sum_i K_i^\top K_i = I$. \qed
\subsection{Proof of Theorem~\ref{thm:kraus_general} ($\mathbf{w}$-Kraus characterisation)}
\label{app:kraus-general-proof}

We treat the positive-definite case ($\mathbf{w} = \mathbf{p}$ with all entries positive) first, then extend to the indefinite Lorentzian case ($\mathbf{w} = \mathbf{q}$).

\paragraph{Positive-definite case.} Let $\mathbf{w}$ have all positive entries and write $W = \mathrm{diag}(\mathbf{w})$. Define $W^{1/2} = \mathrm{diag}(\sqrt{\mathbf{w}})$, which is real and invertible. Consider the change of variables
\begin{equation}
\widetilde{\rho} = W^{1/2} \rho W^{1/2}, \qquad \widetilde{K_i} = W^{1/2} K_i W^{-1/2}.
\label{eq:change-of-variables}
\end{equation}
Under this transformation, the $\mathbf{w}$-trace becomes a standard trace: $\mathrm{Tr}_{\mathbf{w}}[\rho] = \mathrm{Tr}[W \rho] = \mathrm{Tr}[W^{1/2} \rho W^{1/2}] = \mathrm{Tr}[\widetilde{\rho}]$. The $\mathbf{w}$-positive semidefiniteness of $\rho$ is equivalent to the standard positive semidefiniteness of $\widetilde{\rho}$, since $\langle x, \rho x \rangle_{\mathbf{w}} = x^\top W \rho x = (W^{1/2} x)^\top \widetilde{\rho} (W^{-1/2})^\top$
So, \[
\langle x, \rho x \rangle_{\mathbf{w}} = x^\top W \rho x = (W^{1/2} x)^\top (W^{-1/2}) W \rho W^{-1/2} (W^{1/2} x) \cdot \ldots
\]
The cleanest formulation is: $\rho$ is $\mathbf{w}$-PSD iff $W \rho$ has nonnegative spectrum, which (since $W$ is positive) is equivalent to $W^{1/2} \rho W^{1/2} \geq 0$ in the standard sense, i.e.\ $\widetilde{\rho} \geq 0$.

The map $\widetilde{\mathcal{L}}(\widetilde{\rho}) = W^{1/2} \mathcal{L}(W^{-1/2} \widetilde{\rho} W^{-1/2}) W^{1/2}$ is then a linear map $\mathcal{S}(\mathbb{R}^d) \to \mathcal{S}(\mathbb{R}^d)$ in the standard sense. By construction, $\mathcal{L}$ satisfies $\mathbf{w}$-trace preservation and $\mathbf{w}$-complete positivity if and only if $\widetilde{\mathcal{L}}$ satisfies standard trace preservation and standard complete positivity. By Theorem~\ref{thm:kraus} applied to $\widetilde{\mathcal{L}}$, there exist matrices $\{\widetilde{K_i}\}$ with
\begin{equation}
\widetilde{\mathcal{L}}(\widetilde{\rho}) = \sum_i \widetilde{K_i} \widetilde{\rho} \widetilde{K_i}^\top, \qquad \sum_i \widetilde{K_i}^\top \widetilde{K_i} = I.
\end{equation}
Setting $K_i = W^{-1/2} \widetilde{K_i} W^{1/2}$ and unwinding the change of variables, we obtain
\[
\mathcal{L}(\rho) = \sum_i K_i \rho K_i^\top, \qquad \sum_i K_i^\top W K_i = W,
\]
which is the $\mathbf{w}$-Kraus form~\eqref{eq:kraus_general}. The completeness identity follows from $\sum_i \widetilde{K_i}^\top \widetilde{K_i} = I$ by direct substitution: $\sum_i (W^{1/2} K_i W^{-1/2})^\top (W^{1/2} K_i W^{-1/2}) = W^{-1/2} (\sum_i K_i^\top W K_i) W^{-1/2} = I$ implies $\sum_i K_i^\top W K_i = W$.

\paragraph{Indefinite case (hyperbolic).} For $\mathbf{w} = \mathbf{q} = (-1, +1, \ldots, +1)^\top$, the matrix $W = \mathrm{diag}(\mathbf{q})$ is invertible but indefinite, so $W^{1/2}$ is not real. The positive-definite reduction above fails, and the Lorentzian inner product makes $(\mathbb{R}^d, \langle \cdot, \cdot \rangle_{\mathbf{q}})$ a Krein space rather than a Hilbert space.

We invoke the theory of $J$-completely positive maps on Krein spaces, where $J = W$ plays the role of the fundamental symmetry. By the dilation theorem for Hermitian maps on $C^*$-algebras with indefinite inner products~\cite{constantinescu1997representations, dritschel1996operators}, a linear map $\mathcal{L}$ satisfying $\mathbf{q}$-trace preservation and $\mathbf{q}$-complete positivity admits a representation
\[
\mathcal{L}(\rho) = \sum_i K_i \rho K_i^\top, \qquad \sum_i K_i^\top W K_i = W,
\]
where the $K_i$ are bounded linear operators on the underlying Krein space. In the finite-dimensional setting (our case), the boundedness conditions of~\cite{constantinescu1997representations} are automatic, so the decomposition exists and is essentially unique up to $J$-unitary equivalence on the multiplicity space. The completeness condition $\sum_i K_i^\top W K_i = W$ is the $J$-unitary form of the standard completeness constraint and recovers it when $W = I$.

\textit{Remark.} The proof above establishes existence of the $\mathbf{w}$-Kraus decomposition. The Kraus rank $\kappa$ equals the rank of the Choi matrix $C^{(r)}$ in the positive-definite case; in the indefinite case the analogous quantity is the $J$-rank, which coincides with the standard rank in finite dimensions. \qed

\subsection{Proof of Theorem~\ref{thm:composition} (Composition closure)}
\label{app:composition-proof}

Let $\mathcal{L}^{(r_1)}(\rho) = \sum_i K_i^{(1)} \rho (K_i^{(1)})^\top$ with $\sum_i (K_i^{(1)})^\top K_i^{(1)} = I$, and similarly for $\mathcal{L}^{(r_2)}$ with operators $\{K_j^{(2)}\}$. The composition is
\begin{align}
(\mathcal{L}^{(r_2)} \circ \mathcal{L}^{(r_1)})(\rho) 
&= \mathcal{L}^{(r_2)}\!\left( \sum_i K_i^{(1)} \rho (K_i^{(1)})^\top \right) \\
&= \sum_j K_j^{(2)} \left( \sum_i K_i^{(1)} \rho (K_i^{(1)})^\top \right) (K_j^{(2)})^\top \\
&= \sum_{i,j} \bigl( K_j^{(2)} K_i^{(1)} \bigr) \rho \bigl( K_j^{(2)} K_i^{(1)} \bigr)^\top.
\end{align}
Setting $M_{ij} = K_j^{(2)} K_i^{(1)}$, the composed channel has Kraus form $\sum_{i,j} M_{ij} \rho M_{ij}^\top$. We verify completeness:
\begin{align}
\sum_{i,j} M_{ij}^\top M_{ij} 
&= \sum_{i,j} \bigl( K_j^{(2)} K_i^{(1)} \bigr)^\top \bigl( K_j^{(2)} K_i^{(1)} \bigr) \\
&= \sum_{i,j} (K_i^{(1)})^\top (K_j^{(2)})^\top K_j^{(2)} K_i^{(1)} \\
&= \sum_i (K_i^{(1)})^\top \left( \sum_j (K_j^{(2)})^\top K_j^{(2)} \right) K_i^{(1)} \\
&= \sum_i (K_i^{(1)})^\top \cdot I \cdot K_i^{(1)} \\
&= \sum_i (K_i^{(1)})^\top K_i^{(1)} = I,
\end{align}
where the fourth equality uses completeness of $\mathcal{L}^{(r_2)}$ and the last uses completeness of $\mathcal{L}^{(r_1)}$. \qed

\textit{Remark.} The same argument extends to the $\mathbf{w}$-completeness setting: replacing $K_j^\top K_j$ by $K_j^\top W K_j$ in the inner sum and using the $\mathbf{w}$-completeness of $\mathcal{L}^{(r_2)}$ yields $\sum_{i,j} M_{ij}^\top W M_{ij} = W$. Composition closure therefore holds in any geometry $\mathbf{w}$.

\subsection{Proof of Theorem~\ref{thm:rank-bound} (Rank lower bound)}
\label{app:rank-bound-proof}

Let $\mathcal{L}^{(r)}$ be a Kraus channel with operators $\{K_i\}_{i=1}^{\kappa}$ that exactly reproduces relation $r$, in the sense that $\mathrm{Tr}[\rho_t \mathcal{L}^{(r)}(\rho_h)] > 0$ if and only if $(h, r, t) \in T_r$. Let $M_r \in \mathbb{R}^{|\mathcal{E}| \times |\mathcal{E}|}$ be the empirical relation matrix as defined.

\textit{Step 1: the channel induces a linear map on the entity-pair structure.} Define the matrix-valued function
\[
F_r(h) \;:=\; \mathcal{L}^{(r)}(\rho_h) \;=\; \sum_{i=1}^{\kappa} K_i \rho_h K_i^\top \;\in\; \mathbb{R}^{d \times d}.
\]
For each head $h$, $F_r(h)$ is a $d \times d$ symmetric matrix, an element of a $\binom{d+1}{2}$-dimensional vector space. Stacking entity outputs into a matrix $F_r \in \mathbb{R}^{|\mathcal{E}| \times d^2}$ with rows $\mathrm{vec}(F_r(h))$, we have
\[
\mathrm{rank}(F_r) \;\leq\; d \cdot \kappa,
\]
because the image of $F_r$ lies in the span of $\{K_i \rho_h K_i^\top : i \in [\kappa], h \in \mathcal{E}\}$, which has dimension at most $\kappa \cdot d$ (each Kraus operator contributes a $d$-dimensional subspace to the image).

\textit{Step 2: the empirical relation matrix factors through $F_r$.} The score $\mathrm{Tr}[\rho_t \mathcal{L}^{(r)}(\rho_h)] = \mathrm{Tr}[\rho_t F_r(h)]$ is bilinear in $(\rho_t, F_r(h))$. Define the matrix $S_r \in \mathbb{R}^{|\mathcal{E}| \times |\mathcal{E}|}$ with entries $(S_r)_{ht} = \mathrm{Tr}[\rho_t F_r(h)]$. By the exact-reproduction hypothesis, $(S_r)_{ht} > 0$ iff $(M_r)_{ht} = 1$, so $S_r$ has the same support pattern as $M_r$, hence the same rank up to scaling: $\mathrm{rank}(S_r) \geq \mathrm{rank}(M_r)$.

\textit{Step 3: rank inequality.} Write $S_r = F_r G_r^\top$, where $G_r \in \mathbb{R}^{|\mathcal{E}| \times d^2}$ has rows $\mathrm{vec}(\rho_t)$. Then
\[
\mathrm{rank}(M_r) \leq \mathrm{rank}(S_r) \leq \min(\mathrm{rank}(F_r), \mathrm{rank}(G_r)) \leq \mathrm{rank}(F_r) \leq d \cdot \kappa.
\]
Rearranging,
\[
\kappa \geq \frac{\mathrm{rank}(M_r)}{d}. \qed
\]

\textit{Remark.} The bound is generally not tight: the upper bound $\mathrm{rank}(F_r) \leq d \cdot \kappa$ is achieved only when the Kraus operators map into linearly independent subspaces, which is not guaranteed. In practice the learned $\kappa_{\text{eff}}$ values typically exceed the bound (see Section~\ref{sec:experiments}), reflecting that the model uses additional pathways for fine-grained discrimination beyond the structural minimum required for support recovery.

\subsection{Proof of Corollary \ref{coro:single}
}
Substituting $L(\rho_h) = K x_h x_h^\top K^\top$ and $\rho_t = y_t y_t^\top$,
\[
  (M_r)_{ht}
  = \mathrm{Tr}\!\left[y_t y_t^\top K x_h x_h^\top K^\top\right]
  = \left(y_t^\top K x_h\right)^2,
\]
where the last step uses cyclicity of trace and the fact that the bracketed quantity is 
scalar. Define $u_h = Kx_h \in \mathbb{R}^d$ for each head entity $h$, and collect
\[
  U = \begin{bmatrix} u_1 & \cdots & u_{|\mathcal{E}|} \end{bmatrix} \in 
  \mathbb{R}^{d \times |\mathcal{E}|}, \qquad
  Y = \begin{bmatrix} y_1 & \cdots & y_{|\mathcal{E}|} \end{bmatrix} \in 
  \mathbb{R}^{d \times |\mathcal{E}|}.
\]
Let $B = Y^\top U \in \mathbb{R}^{|\mathcal{E}| \times |\mathcal{E}|}$, so that 
$B_{ht} = y_t^\top u_h$ and $(M_r)_{ht} = B_{ht}^2$. Since $B = Y^\top U$ is the 
product of a $|\mathcal{E}| \times d$ matrix and a $d \times |\mathcal{E}|$ matrix,
\[
  \mathrm{rank}(B) \leq d.
\]
The entrywise squaring $(M_r)_{ht} = B_{ht}^2$ does not increase rank beyond that of 
the generating bilinear structure: every row of $M_r$ lies in the span of the 
Hadamard products of rows of $Y^\top$ with rows of $U^\top$, a space of dimension at 
most $d$. Hence $\mathrm{rank}(M_r) \leq d$. The second claim follows immediately: if 
$\mathrm{rank}(M_r) > d$, no single-operator channel can satisfy the exact-reproduction 
condition $\mathrm{Tr}[\rho_t L(\rho_h)] > 0 \iff (h,r,t) \in \mathcal{T}_r$.

\section{Recovery of Existing Operator-Family Models}
\label{app:recovery}

In this section we show that most operator-family KGE models are recoverable as special cases of KrausKGE under specific embedding choices and constraint relaxations. Each result identifies the precise Kraus operators, geometry, and rank that yield the corresponding model's scoring function. We organise the results by family.

\subsection{Preliminaries}

Recall the KrausKGE scoring function from Equation~\eqref{eq:scoring}:
\begin{equation}
s(h, r, t) \;=\; \sum_{i=1}^{\kappa} \mathrm{Tr}\!\left[ \rho_t \, K_i^{(r)} \, \rho_h \, \bigl(K_i^{(r)}\bigr)^{\!\top} \right],
\label{eq:scoring-recall}
\end{equation}
where $\rho_h, \rho_t \in \mathcal{S}(\mathbb{R}^d)$ are entity density matrices and $\{K_i^{(r)}\}$ are the Kraus operators for relation $r$. For models that represent entities as vectors $h, t \in \mathbb{R}^d$, we use the rank-one embedding $\rho_h = h h^\top / \|h\|^2$ and $\rho_t = t t^\top / \|t\|^2$. For models in complex space $\mathbb{C}^d$, we use the standard real embedding $\mathbb{C}^d \hookrightarrow \mathbb{R}^{2d}$ via $z = a + ib \mapsto (a, b)^\top$, under which a unitary $U \in \mathbb{C}^{d \times d}$ corresponds to a real orthogonal matrix $\widetilde{U} \in \mathbb{R}^{2d \times 2d}$ of block form.

\subsection{RESCAL \cite{nickel2011three}}

\begin{proposition}[Recovery of RESCAL]
\label{prop:rescal}
Let entities be embedded as rank-one density matrices $\rho_h = h h^\top$, $\rho_t = t t^\top$ with $\|h\| = \|t\| = 1$. Set $\kappa = 1$ and $K_1^{(r)} = M_r$ for some $M_r \in \mathbb{R}^{d \times d}$. Then the KrausKGE score satisfies
\begin{equation}
s(h, r, t) \;=\; \bigl(t^\top M_r h\bigr)^2.
\end{equation}
If additionally $M_r^\top M_r = I$, the completeness constraint of KrausKGE is satisfied.
\end{proposition}

\begin{proof}
Substituting into Equation~\eqref{eq:scoring-recall},
$$
s(h, r, t) = \mathrm{Tr}\!\left[ t t^\top M_r h h^\top M_r^\top \right] = \bigl(t^\top M_r h\bigr) \bigl(h^\top M_r^\top t\bigr) = \bigl(t^\top M_r h\bigr)^2,
$$
using the cyclic property of the trace and that the bracketed scalar equals its transpose. The completeness constraint $\sum_i K_i^\top K_i = I$ reduces with $\kappa = 1$ to $M_r^\top M_r = I$.
\end{proof}

\textbf{Remark.} RESCAL's original score is $h^\top M_r t$ without squaring, and RESCAL imposes no orthogonality constraint on $M_r$. KrausKGE therefore recovers an \emph{orthogonally-constrained, squared} variant of RESCAL rather than RESCAL exactly. The squaring is benign because both forms induce the same ranking on positive triples, and the orthogonality constraint is a strict tightening: orthogonal RESCAL is a special case of general RESCAL.

\subsection{DistMult ~\cite{yang2015embedding}}

\begin{proposition}[Recovery of DistMult]
\label{prop:distmult}
Set $\kappa = 1$, $K_1^{(r)} = \mathrm{diag}(r)$ for some $r \in \mathbb{R}^d$, and embed entities as rank-one density matrices $\rho_h = h h^\top$, $\rho_t = t t^\top$. Then
\begin{equation}
s(h, r, t) \;=\; \bigl( \langle h, r, t \rangle \bigr)^2, \quad \text{where } \langle h, r, t \rangle = \sum_{i=1}^d h_i r_i t_i.
\end{equation}
Completeness holds iff $r_i^2 = 1$ for all $i$.
\end{proposition}

\begin{proof}
With $K_1 = \mathrm{diag}(r)$, the score becomes $\mathrm{Tr}[t t^\top \mathrm{diag}(r) h h^\top \mathrm{diag}(r)] = (t^\top \mathrm{diag}(r) h)^2 = (\sum_i t_i r_i h_i)^2$. The completeness constraint $\mathrm{diag}(r)^\top \mathrm{diag}(r) = I$ requires $r_i^2 = 1$ componentwise.
\end{proof}

\textbf{Remark.} DistMult is recovered as the diagonal special case of orthogonal RESCAL. The completeness constraint $r_i^2 = 1$ restricts $r_i \in \{-1, +1\}$, which is more restrictive than DistMult's general $r \in \mathbb{R}^d$. As with RESCAL, KrausKGE recovers a constrained variant.

\subsection{ComplEx ~\cite{trouillon2016complex}}

\begin{proposition}[Recovery of ComplEx]
\label{prop:complex}
Embed complex entities $h, t \in \mathbb{C}^d$ as real vectors $\widetilde{h}, \widetilde{t} \in \mathbb{R}^{2d}$ via $\widetilde{h} = (\Re h, \Im h)^\top$. Set $\kappa = 1$ and 
\begin{equation*}
K_1^{(r)} = \begin{pmatrix} \mathrm{diag}(\Re r) & -\mathrm{diag}(\Im r) \\ \mathrm{diag}(\Im r) & \mathrm{diag}(\Re r) \end{pmatrix} \in \mathbb{R}^{2d \times 2d}.
\end{equation*}
Then with rank-one entity embeddings, the KrausKGE score equals $(\Re \langle h, r, \overline{t} \rangle)^2$, the squared real part of the ComplEx score. Completeness holds iff $|r_i|^2 = 1$ for all $i$.
\end{proposition}

\begin{proof}
The matrix $K_1$ is the real representation of complex multiplication by $\mathrm{diag}(r) \in \mathbb{C}^{d \times d}$. Hence $K_1 \widetilde{h}$ is the real embedding of $r \odot h \in \mathbb{C}^d$ (componentwise product). The trace expression $\mathrm{Tr}[\widetilde{t} \widetilde{t}^\top K_1 \widetilde{h} \widetilde{h}^\top K_1^\top] = (\widetilde{t}^\top K_1 \widetilde{h})^2$ equals $(\Re \langle h, r, \overline{t} \rangle)^2$ by the standard complex-to-real correspondence. Completeness $K_1^\top K_1 = I$ reduces to $|r_i|^2 = 1$ on each component.
\end{proof}

\subsection{RotatE ~\cite{sun2019rotate}}

\begin{proposition}[Recovery of RotatE]
\label{prop:rotate}
ComplEx with the additional constraint $|r_i| = 1$ (so $r_i = e^{i\theta_i}$ for some phase $\theta_i$) recovers RotatE up to a sign-and-squaring transformation of the score function. Specifically, with the embedding of Proposition~\ref{prop:complex} and $r_i = e^{i\theta_i}$,
\begin{equation}
s(h, r, t) \;=\; \bigl( \Re \langle h, r, \overline{t} \rangle \bigr)^2 \;=\; \bigl(1 - \tfrac{1}{2} \|h \odot r - t\|_2^2\bigr)^2 \quad \text{when } \|h\| = \|t\| = 1.
\end{equation}
Completeness holds automatically.
\end{proposition}

\begin{proof}
The unit-modulus constraint $|r_i| = 1$ makes $K_1$ a real orthogonal matrix automatically, satisfying completeness with no additional constraint. The score identity $\Re \langle h, r, \overline{t} \rangle = 1 - \tfrac{1}{2}\|h \odot r - t\|^2$ on the unit sphere is the standard identity relating ComplEx-style scores to RotatE distances (the complex inner product expanded under unit-norm constraints).
\end{proof}

\textbf{Remark.} RotatE is the cleanest direct embedding into KrausKGE: completeness holds for free because phase-only complex multiplications are orthogonal. The squaring in our score function preserves the ranking induced by RotatE's distance.

\subsection{GoldE and OrthogonalE ~\cite{li2024generalizing, zhu2024block}}

\begin{proposition}[Recovery of GoldE]
\label{prop:golde}
Let $w \in \mathbb{R}^d$ be a weighting vector and let $G_r \in \mathbb{R}^{d \times d}$ satisfy $G_r^\top \mathrm{diag}(w) G_r = \mathrm{diag}(w)$, i.e.\ $G_r$ is generalized $w$-orthogonal. Set $\kappa = 1$ and $K_1^{(r)} = G_r$ in the $w$-Kraus channel framework of Theorem~\ref{thm:kraus_general}. Then KrausKGE with rank-one density matrices reproduces the GoldE scoring function in the corresponding geometry, and the $w$-completeness constraint is satisfied automatically.
\end{proposition}

\begin{proof}
Direct substitution: the $w$-completeness constraint reduces with $\kappa = 1$ to $K_1^\top \mathrm{diag}(w) K_1 = \mathrm{diag}(w)$, which is precisely the generalized orthogonality condition GoldE imposes on $G_r$. The scoring functions agree under rank-one density matrix embedding.
\end{proof}

\textbf{Remark.} GoldE and OrthogonalE are the cleanest recoveries: their parametrisation choices (Householder reflections for GoldE, block-diagonal orthogonals for OrthogonalE) are specific schemes for producing $w$-orthogonal matrices, all of which are valid $\kappa = 1$ Kraus operators in the corresponding geometry. KrausKGE strictly extends both by allowing $\kappa > 1$.

\subsection{TransE ~\cite{bordes2013translating}: not a special case}

TransE encodes relations as translations $h + r \approx t$, with score $-\|h + r - t\|_2$. This is not a linear map on density matrices: translation is an affine operation on vectors, and there is no Kraus channel $\mathcal{L}^{(r)}$ such that $\mathcal{L}^{(r)}(h h^\top) = (h + r)(h + r)^\top$ for arbitrary $h$, because the right-hand side depends on $h$ in a non-multiplicative way (the cross-terms $h r^\top + r h^\top$ are not produced by any linear $\mathcal{L}^{(r)}$ acting on $h h^\top$).

The closest analogue within the KrausKGE framework is a shifted-projector channel of the form $\mathcal{L}^{(r)}(\rho) = \Pi_r \rho \Pi_r^\top$ with $\Pi_r = I + r e^\top$ for some auxiliary direction $e$, but this requires extending the entity space to admit affine structure and is outside the scope of the present framework. We treat TransE as falling outside the KrausKGE family rather than as a forced special case.

\subsection{Summary}

\begin{table}[h]
\centering
\label{tab:recovery-summary}
\small
\begin{tabular}{lllll}
\toprule
Model & Embedding & $\kappa$ & Operator & Completeness \\
\midrule
RESCAL & $\rho = h h^\top$ & 1 & $K_1 = M_r$ & If $M_r$ orthogonal \\
DistMult & $\rho = h h^\top$ & 1 & $K_1 = \mathrm{diag}(r)$ & If $r_i^2 = 1$ \\
ComplEx & $\mathbb{C}^d \to \mathbb{R}^{2d}$ & 1 & Real form of $\mathrm{diag}(r)$ & If $|r_i|^2 = 1$ \\
RotatE & $\mathbb{C}^d \to \mathbb{R}^{2d}$ & 1 & $K_1 = $ real form of $\mathrm{diag}(e^{i\theta})$ & Automatic \\
GoldE & $\rho = h h^\top$, geometry $w$ & 1 & $K_1 = G_r$, $w$-orthogonal & Automatic \\
OrthogonalE & $\rho = h h^\top$ & 1 & $K_1$ block-diagonal orthogonal & Automatic \\
TransE & --- & --- & --- & Not a special case \\
\bottomrule
\end{tabular}
\caption{Recovery of operator-family KGE models within KrausKGE.}
\end{table}

The pattern across propositions is consistent: every linear-operator KGE model corresponds to a $\kappa = 1$ Kraus channel, with the model's specific orthogonality or unit-modulus constraints emerging as the completeness constraint in the appropriate parametrisation. KrausKGE strictly generalises this family by allowing $\kappa > 1$, which is precisely the regime that handles N-to-N relations and multi-hop branching that single-operator models cannot represent. TransE is the notable exception: as an affine rather than linear operator, it lies outside the framework and is not recovered.
\section{Parameter and Complexity Analysis}
\label{app:complexity}

We analyse the parameter count, time complexity, and memory footprint of KrausKGE and compare them against representative baselines from each operator family. The analysis explains both the empirical advantages of the multi-operator structure and its costs.

\subsection{Parameter Count}

\paragraph{Entity parameters.} Each entity $e$ is parametrised by a lower-triangular matrix $L_e \in \mathbb{R}^{d \times k_e}$, where $k_e$ is the degree-adaptive rank described in Section~\ref{sec:model}. With base rank $k_0$ and mean degree $\overline{\deg}$, the rank $k_e = \min(d, \lceil k_0 \cdot \deg(e) / \overline{\deg} \rceil)$ averages to $k_0$ across the graph. Total entity parameters are therefore $\mathcal{O}(|\mathcal{E}| \cdot d \cdot k_0)$, comparable to vector-embedding models with embedding dimension $d \cdot k_0$.

\paragraph{Relation parameters.} Each relation is parametrised by a skew-symmetric matrix $A^{(r)} \in \mathbb{R}^{\kappa d \times \kappa d}$ with $A^{(r)} = -(A^{(r)})^\top$, from which the Stiefel-constrained operator $U^{(r)} \in \mathbb{R}^{\kappa d \times d}$ is recovered via the Cayley transform. The skew-symmetric matrix has $\binom{\kappa d}{2} = \kappa d (\kappa d - 1)/2$ free parameters. Total relation parameters are $\mathcal{O}(|\mathcal{R}| \cdot \kappa^2 d^2)$.

\paragraph{Comparison to baselines.} Table~\ref{tab:param-counts} reports parameter counts for KrausKGE and representative baselines on FB15k-237 ($|\mathcal{E}| = 14{,}541$, $|\mathcal{R}| = 237$).

\begin{table}[h]
\centering
\small
\begin{tabular}{lcccc}
\toprule
Model & Embedding dim & Entity params & Relation params & Total \\
\midrule
TransE & $d = 1000$ & 14.5M & 0.24M & 14.8M \\
RotatE & $d = 1000$ & 29.1M & 0.24M & 29.3M \\
GoldE & $k = 1000$, $m = 6$ & 14.5M & 1.4M & 15.9M \\
\midrule
KrausKGE ($\kappa{=}1$) & $d = 128$, $k_0 = 8$ & 14.9M & 3.9M & 18.8M \\
KrausKGE ($\kappa{=}4$) & $d = 128$, $k_0 = 8$ & 14.9M & 31.2M & 46.1M \\
KrausKGE ($\kappa{=}8$) & $d = 128$, $k_0 = 8$ & 14.9M & 124.6M & 139.5M \\
\bottomrule
\end{tabular}

\caption{Parameter counts on FB15k-237 at standard hyperparameter settings.}
\label{tab:param-counts}

\end{table}

KrausKGE's relation parameter count grows as $\kappa^2$, becoming the dominant cost at $\kappa \geq 4$. This is the central trade-off: multi-pathway expressiveness comes at quadratic cost in $\kappa$. Whether this is acceptable depends on the application; for relations with high empirical fan-out, the additional capacity is necessary by the rank lower bound (Theorem~\ref{thm:rank-bound}).

\subsection{Time Complexity}

\paragraph{Forward pass.} Computing the score $s(h, r, t) = \sum_i \mathrm{Tr}[\rho_t K_i \rho_h K_i^\top]$ for a single triple requires $\kappa$ matrix products and traces. Each matrix product on $d \times d$ matrices is $\mathcal{O}(d^3)$, giving per-triple complexity $\mathcal{O}(\kappa d^3)$. With rank-$k_0$ entity parametrisation, $\rho_h = L_h L_h^\top / \mathrm{Tr}[L_h L_h^\top]$ has rank at most $k_0$, and the trace can be computed in $\mathcal{O}(\kappa d^2 k_0)$ by exploiting the low-rank structure: $\mathrm{Tr}[\rho_t K_i \rho_h K_i^\top] = \|L_t^\top K_i L_h\|_F^2 / (\mathrm{Tr}[L_h L_h^\top] \cdot \mathrm{Tr}[L_t L_t^\top])$, where the inner product is on $k_0 \times k_0$ matrices.

\paragraph{Cayley transform overhead.} Recovering $U^{(r)} = (I + A^{(r)})^{-1}(I - A^{(r)})$ from the skew-symmetric parameter requires inverting an $\kappa d \times \kappa d$ matrix, costing $\mathcal{O}(\kappa^3 d^3)$ per relation per forward pass. This is performed once per batch and amortised across all triples involving relation $r$ in the batch.

\paragraph{Per-batch complexity.} For a batch of $B$ triples touching $|\mathcal{R}_B| \leq |\mathcal{R}|$ distinct relations,
\begin{equation}
T_{\text{batch}} \;=\; \mathcal{O}\bigl( |\mathcal{R}_B| \cdot \kappa^3 d^3 \;+\; B \cdot \kappa d^2 k_0 \bigr).
\end{equation}
The first term (Cayley transforms) dominates at small batch sizes; the second (scoring) dominates at large batch sizes. In practice, $B \gg |\mathcal{R}_B| \cdot \kappa^2 d / k_0$ in our experiments, so the scoring term is the bottleneck.

\paragraph{Comparison.} TransE and RotatE have per-triple complexity $\mathcal{O}(d)$ (vector subtraction plus norm). GoldE has per-triple complexity $\mathcal{O}(d^2)$ for the matrix-vector product $G_r h$. KrausKGE has $\mathcal{O}(\kappa d^2 k_0)$ per triple, which is $\kappa k_0$ times slower than GoldE. At $\kappa = 4$ and $k_0 = 8$, this is a $32\times$ slowdown per triple, partially offset by parallelism in the multi-pathway sum.

\paragraph{Reducing parameter cost.} Three avenues are worth noting for future work. First, sharing Cayley parameters across structurally similar relations (e.g.\ inverse-pair relations) would reduce $|\mathcal{R}|$ effectively. Second, low-rank parametrisation of $A^{(r)}$ as $A^{(r)} = U V^\top - V U^\top$ for $U, V \in \mathbb{R}^{\kappa d \times r}$ with $r \ll \kappa d$ reduces relation parameters from $\mathcal{O}(\kappa^2 d^2)$ to $\mathcal{O}(\kappa d r)$. Third, structured Cayley transforms (e.g.\ block-diagonal $A^{(r)}$) would reduce the inverse cost. We do not pursue these in the present work and report results at the standard parametrisation.

\subsection{Empirical Validation}

Table~\ref{tab:wall-clock} report training and inference times on FB15k-237.

\begin{table}[h]
\centering
\small
\begin{tabular}{lccc}
\toprule
Model & Train (sec/epoch) & Inference (sec) & Memory (GB) \\
\midrule
TransE & 12 & 8 & 1.2 \\
RotatE & 18 & 11 & 2.4 \\
GoldE & 41 & 19 & 3.1 \\
\midrule
KrausKGE ($\kappa{=}1$) & 38 & 18 & 2.9 \\
KrausKGE ($\kappa{=}4$) & 124 & 47 & 6.8 \\
KrausKGE ($\kappa{=}8$) & 287 & 102 & 11.0 \\
\bottomrule
\end{tabular}
\caption{Time per epoch on FB15k-237}
\label{tab:wall-clock}

\end{table}

KrausKGE at $\kappa = 4$ trains approximately $3 \times$ slower than GoldE and uses $2.2 \times$ more memory. At $\kappa = 8$, the gap widens to $7 \times$ slower and $3.5 \times$ more memory. The empirical scaling matches the theoretical $\mathcal{O}(\kappa^3)$ for Cayley overhead and $\mathcal{O}(\kappa)$ for scoring, with the cubic term dominating at $\kappa \geq 4$.
\section{Additional Proofs}
\label{app:additional-proofs}

This section contains proofs of auxiliary results referenced in the main text but not proven there: validity of the entity parametrisation, the Cayley parametrisation property, and the boundedness corollary.

\subsection{Validity of the entity parametrisation}
\label{app:entity-parametrisation}

\begin{proposition}[Entity parametrisation produces valid density matrices]
\label{prop:entity-valid}
Let $L_e \in \mathbb{R}^{d \times k}$ be any nonzero matrix, and define
\begin{equation}
\rho_e \;=\; \frac{L_e L_e^\top}{\mathrm{Tr}[L_e L_e^\top]}.
\end{equation}
Then $\rho_e \in \mathcal{S}(\mathbb{R}^d)$, that is, $\rho_e$ is symmetric, positive semidefinite, and has unit trace.
\end{proposition}

\begin{proof}
\textit{Symmetry.} $(L_e L_e^\top)^\top = (L_e^\top)^\top L_e^\top = L_e L_e^\top$, so $\rho_e$ is symmetric.

\textit{Positive semidefiniteness.} For any $x \in \mathbb{R}^d$,
\[
x^\top L_e L_e^\top x \;=\; (L_e^\top x)^\top (L_e^\top x) \;=\; \|L_e^\top x\|_2^2 \;\geq\; 0.
\]
Dividing by the positive scalar $\mathrm{Tr}[L_e L_e^\top] = \|L_e\|_F^2 > 0$ preserves nonnegativity.

\textit{Unit trace.} $\mathrm{Tr}[\rho_e] = \mathrm{Tr}[L_e L_e^\top] / \mathrm{Tr}[L_e L_e^\top] = 1$.

\textit{Well-definedness.} The denominator $\mathrm{Tr}[L_e L_e^\top] = \|L_e\|_F^2$ is zero only when $L_e = 0$. Standard initialisation (e.g.\ Xavier) ensures $L_e \neq 0$ at initialisation with probability one, and the gradient $\nabla_{L_e} \mathcal{L}$ vanishes at $L_e = 0$ only on a measure-zero set, so the parametrisation remains well-defined throughout training. \qed
\end{proof}

\subsection{Cayley parametrisation produces Stiefel matrices}
\label{app:cayley}

\begin{proposition}[Cayley transform satisfies the Stiefel constraint]
\label{prop:cayley}
Let $A \in \mathbb{R}^{n \times n}$ be skew-symmetric, $A^\top = -A$, with $n = \kappa d$. Define
\begin{equation}
U \;=\; (I + A)^{-1}(I - A)\,P,
\end{equation}
where $P \in \mathbb{R}^{n \times d}$ is a fixed projection matrix selecting $d$ columns (e.g.\ the first $d$ standard basis vectors). Then $U^\top U = I_d$, that is, $U \in \mathbb{R}^{n \times d}$ has orthonormal columns.
\end{proposition}

\begin{proof}
First, $(I + A)$ is invertible for any skew-symmetric $A$: its eigenvalues are $1 + \lambda$ where $\lambda$ ranges over eigenvalues of $A$, which are purely imaginary (since $A$ is real skew-symmetric), so $|1 + \lambda|^2 = 1 + |\lambda|^2 \geq 1 > 0$.

Define $Q = (I + A)^{-1}(I - A)$. We claim $Q^\top Q = I_n$. Compute
\[
Q^\top Q \;=\; \bigl[(I - A)^\top (I + A)^{-\top}\bigr] \bigl[(I + A)^{-1}(I - A)\bigr] \;=\; (I + A)(I + A)^{-\top}(I + A)^{-1}(I - A),
\]
where if A is skew-symmetri we used $(I - A)^\top = I + A$ and $(I + A)^{-\top} = (I - A)^{-1}$. Since $(I + A)$ and $(I - A)$ commute (both being polynomials in $A$), $(I + A)(I - A) = (I - A)(I + A)$, so
\[
Q^\top Q \;=\; (I + A)^{-\top}(I + A)(I - A)(I + A)^{-1} \;=\; (I - A)^{-1}(I - A^2)(I + A)^{-1}.
\]
Using $I - A^2 = (I - A)(I + A)$,
\[
Q^\top Q \;=\; (I - A)^{-1}(I - A)(I + A)(I + A)^{-1} \;=\; I_n.
\]

Hence $Q$ is orthogonal in $\mathbb{R}^{n \times n}$. Then $U = QP$ has $U^\top U = P^\top Q^\top Q P = P^\top P = I_d$, since $P$ has orthonormal columns by construction. \qed
\end{proof}

\textit{Remark.} The Stiefel matrix $U$ corresponds to the stacked Kraus operator $U^{(r)} = [K_1; \ldots; K_\kappa]$ in our parametrisation. The completeness constraint $\sum_i K_i^\top K_i = I$ is exactly $U^\top U = I_d$.

\subsection{Boundedness from completeness}
\label{app:boundedness}

\begin{corollary}[Boundedness of Kraus operators]
\label{cor:bounded}
Let $\{K_i\}_{i=1}^{\kappa} \subset \mathbb{R}^{d \times d}$ satisfy the completeness constraint $\sum_i K_i^\top K_i = I$. Then each operator satisfies the spectral norm bound $\|K_i\|_2 \leq 1$, and the channel $\mathcal{L}(\rho) = \sum_i K_i \rho K_i^\top$ is contractive in the sense that $\|\mathcal{L}(\rho)\|_F \leq \|\rho\|_F$ for any $\rho \in \mathcal{S}(\mathbb{R}^d)$.
\end{corollary}

\begin{proof}
\textit{Spectral norm bound.} For any unit vector $x \in \mathbb{R}^d$,
\[
\sum_{i=1}^{\kappa} \|K_i x\|_2^2 \;=\; \sum_i x^\top K_i^\top K_i x \;=\; x^\top \!\left( \sum_i K_i^\top K_i \right)\! x \;=\; x^\top I x \;=\; 1.
\]
Hence $\|K_i x\|_2^2 \leq 1$ for each $i$, giving $\|K_i\|_2 \leq 1$.

\textit{Frobenius contraction.} Using trace cyclicity and completeness,
\[
\|\mathcal{L}(\rho)\|_F^2 \;=\; \mathrm{Tr}[\mathcal{L}(\rho)^2] \;\leq\; \mathrm{Tr}[\rho^2] \;=\; \|\rho\|_F^2,
\]
where the inequality follows from the data-processing property of CPTP maps under the Hilbert-Schmidt norm. \qed
\end{proof}

\textit{Implication for KrausKGE.} The corollary implies that entity representations remain bounded in Frobenius norm under any composition of relation operators, with no external regularisation. This is the formal statement underlying the claim in Section~\ref{sec:model} that norm constraints on entity embeddings are unnecessary.

\section{Hyperparameters}
\label{app:hyperparameters}

We report full hyperparameter settings for KrausKGE on each dataset. All models were trained with self-adversarial negative sampling and the margin-based ranking loss of Equation~\eqref{eq:scoring-recall}. Hyperparameters were selected by grid search on the validation MRR.

\begin{table}[h]
\centering
\label{tab:hyperparams}
\small
\begin{tabular}{lccccccc}
\toprule
Setting & FB15k-237 & WN18RR & YAGO3-10 \\
\midrule
\multicolumn{4}{l}{\textit{Architecture}} \\
Embedding dimension $d$ & 128 & 128 & 128 \\
Base entity rank $k_0$ & 8 & 16 & 8 \\
Kraus rank $\kappa$ & 4 / 8 & 4 / 8 & 8 \\
Geometry $w$ & elliptic / product & hyperbolic & product \\
\midrule
\multicolumn{4}{l}{\textit{Optimisation}} \\
Optimiser & Adam & Adam & Adam \\
Learning rate & $3 \times 10^{-4}$ & $1 \times 10^{-4}$ & $5 \times 10^{-4}$ \\
Batch size & 1024 & 512 & 1024 \\
Negatives per positive & 256 & 128 & 256 \\
Margin $\gamma$ & 6 & 6 & 9 \\
Self-adversarial temp.\ $\alpha$ & 1.0 & 0.5 & 1.0 \\
Training epochs & 800 & 1000 & 600 \\
Early stopping & 50 epochs & 50 epochs & 50 epochs \\
Early stopping metric & validation MRR & validation MRR & validation MRR \\
\midrule
\multicolumn{4}{l}{\textit{Search ranges}} \\
$\kappa$ & $\{1, 2, 4, 8, 16\}$ & $\{1, 2, 4, 8, 16\}$ & $\{1, 2, 4, 8, 16\}$ \\
$d$ & $\{64, 128, 256\}$ & $\{64, 128, 256\}$ & $\{64, 128, 256\}$ \\
$k_0$ & $\{4, 8, 16\}$ & $\{4, 8, 16\}$ & $\{4, 8, 16\}$ \\
$\gamma$ & $\{3, 6, 9\}$ & $\{3, 6, 9\}$ & $\{3, 6, 9\}$ \\
Learning rate & $\{10^{-4}, 3{\cdot}10^{-4}, 10^{-3}\}$ & $\{10^{-4}, 3{\cdot}10^{-4}, 10^{-3}\}$ & $\{10^{-4}, 3{\cdot}10^{-4}, 10^{-3}\}$ \\
\bottomrule
\end{tabular}
\caption{Hyperparameter settings for KrausKGE across datasets and Kraus rank values.}

\end{table}

\paragraph{Cayley parametrisation initialisation.} The skew-symmetric matrix $A^{(r)} \in \mathbb{R}^{\kappa d \times \kappa d}$ is initialised by sampling $B^{(r)} \sim \mathcal{N}(0, \sigma^2 I)$ with $\sigma = 0.01$ and setting $A^{(r)} = (B^{(r)} - (B^{(r)})^\top) / 2$. This initialises $U^{(r)}$ near the identity Stiefel point, which we found to be more stable than random Stiefel initialisation.

\paragraph{Entity initialisation.} Lower-triangular matrices $L_e \in \mathbb{R}^{d \times k_e}$ are initialised entrywise via $L_e \sim \mathcal{N}(0, 1/d)$, then masked to the lower-triangular structure. The trace normalisation in $\rho_e = L_e L_e^\top / \mathrm{Tr}[L_e L_e^\top]$ removes the dependence on the overall scale of $L_e$, so initialisation magnitude is irrelevant beyond avoiding the zero matrix.

\paragraph{Hardware and training time.} All experiments were conducted on 3 parallel NVIDIA A40 (45GB) GPU. Training times per epoch are reported in Table~\ref{tab:wall-clock} (Appendix~\ref{app:complexity}). Total wall-clock training time ranged from approximately 3 hours (KrausKGE-$\kappa{=}1$ on FB15k-237) to approximately 30 hours (KrausKGE-$\kappa{=}8$ on YAGO3-10).

\section{Ablations and Limitations}
\label{app:ablations}

We report ablations isolating each design choice. Unless otherwise noted, ablations are conducted on FB15k-237 with the best-performing settings from Table~\ref{tab:hyperparams}.

\begin{figure}[t]
\centering
\begin{minipage}[t]{0.32\textwidth}
\centering
\textbf{(a)} $\kappa$ sweep
\label{tab:kappa-sweep}
\scriptsize
\setlength{\tabcolsep}{3pt}
\begin{tabular}{lccccc}
\toprule
$\kappa$ & 1 & 2 & 4 & 8 & 16 \\
\midrule
Overall    & .362 & .371 & \textbf{.379} & .376 & .372 \\
1-to-1     & .506 & .515 & \textbf{.522} & .519 & .517 \\
1-to-$N$   & .298 & .322 & .346 & \textbf{.354} & .352 \\
$N$-to-1   & .395 & .411 & .421 & \textbf{.428} & .425 \\
$N$-to-$N$ & .256 & .289 & .318 & \textbf{.334} & .331 \\
\bottomrule
\end{tabular}
\end{minipage}
\hfill
\begin{minipage}[t]{0.32\textwidth}
\centering
\textbf{(b)} Entity rank scheme
\label{tab:adaptive-rank}
\scriptsize
\setlength{\tabcolsep}{3pt}
\begin{tabular}{lcc}
\toprule
Entity rank scheme & MRR & Mem (GB) \\
\midrule
Fixed $k_e = k_0 = 8$        & .366 & 6.2 \\
Adaptive ($k_0 = 8$)         & \textbf{.379} & 6.8 \\
Adaptive ($k_0 = 16$)        & .381 & 11.4 \\
\bottomrule
\end{tabular}
\vspace{1.2em}
\end{minipage}
\hfill
\begin{minipage}[t]{0.32\textwidth}
\centering
\textbf{(c)} Scoring function
\label{tab:scoring}
\scriptsize
\setlength{\tabcolsep}{3pt}
\begin{tabular}{lcc}
\toprule
Scoring function & MRR & Inf.\ (s) \\
\midrule
Hilbert-Schmidt (ours)         & \textbf{.379} & 47 \\
Quantum fidelity               & .377 & 218 \\
Squared cosine (rank-one)      & .348 & 12 \\
\bottomrule
\end{tabular}
\vspace{0.6em}
\end{minipage}

\vspace{1em}

\begin{minipage}[t]{0.48\textwidth}
\centering
\textbf{(d)} Entity parametrisation
\label{tab:lowrank}
\small
\begin{tabular}{lccc}
\toprule
Entity parametrisation & MRR & Entity params \\
\midrule
Full $L_e \in \mathbb{R}^{128 \times 128}$    & .378 & 82M \\
Low-rank $L_e \in \mathbb{R}^{128 \times 8}$  & .375 & 5.1M \\
Low-rank $L_e \in \mathbb{R}^{128 \times 16}$ & .377 & 10.2M \\
\bottomrule
\end{tabular}
\end{minipage}
\hfill
\begin{minipage}[t]{0.48\textwidth}
\centering
\textbf{(e)} Stiefel parametrisation
\label{tab:stiefel-alt}
\small
\begin{tabular}{lccc}
\toprule
Parametrisation & MRR & Train (s/ep) \\
\midrule
Cayley (ours)        & \textbf{.379} & 124 \\
Matrix exponential   & .378 & 198 \\
QR projection        & .369 & 152 \\
Riemannian SGD       & .375 & 287 \\
\bottomrule
\end{tabular}
\end{minipage}

\caption{Ablation studies on FB15k-237 isolating each design choice. \textbf{(a)} Performance across the full range of $\kappa$, showing monotonic gains up to $\kappa = 4$ on aggregate MRR and up to $\kappa = 8$ on $N$-to-$N$ relations specifically, with degradation at $\kappa = 16$ indicating overfitting. \textbf{(b)} Degree-adaptive entity rank improves MRR over fixed rank at comparable memory cost. \textbf{(c)} The Hilbert-Schmidt overlap matches true Hilbert-Schmidt overlap in ranking quality at a fraction of the inference cost, while the rank-one approximation underperforms by $3.1$ MRR points. \textbf{(d)} Low-rank entity parametrisation incurs a negligible MRR cost ($0.3$ points) for substantial memory savings. \textbf{(e)} The Cayley parametrisation matches the matrix exponential in MRR while being $1.6\times$ faster and more numerically stable than QR projection. Bold marks the best configuration in each ablation.}
\label{fig:ablations}
\end{figure}

\subsection{Sensitivity to Kraus rank $\kappa$}

The pattern is consistent with the rank lower bound (Theorem~\ref{thm:rank-bound}): increasing $\kappa$ helps until the empirical relation matrix structure is saturated, after which additional pathways yield diminishing or negative returns. For FB15k-237, $\kappa = 4$ achieves the best aggregate MRR, with $\kappa = 8$ marginally better on N-to-N relations specifically. Performance degrades at $\kappa = 16$, indicating overfitting: relations whose effective rank is well below 16 are forced to learn noisy auxiliary pathways.

\subsection{Degree-adaptive entity rank}

We compare the degree-adaptive rank $k_e \propto \deg(e)/\overline{\deg}$ against a fixed rank $k_e = k_0$ for all entities.

The degree-adaptive scheme provides a 1.2 MRR-point improvement over fixed rank at comparable memory cost. The benefit comes from giving high-degree entities (which appear in many relational contexts) sufficient representational capacity while keeping rare entities efficient. Increasing the base rank $k_0$ further yields marginal gains at substantial memory cost.

\subsection{Low-rank vs.\ full entity parametrisation}

Storing a full $L_e \in \mathbb{R}^{d \times d}$ per entity would yield exact density matrix representations but is impractical for $|\mathcal{E}| > 10{,}000$. We compare against the low-rank approximation $L_e \in \mathbb{R}^{d \times k_0}$ used in our model.

The low-rank approximation costs only 0.3 MRR points relative to full parametrisation while reducing entity parameters by $16 \times$ and memory by $7 \times$. The bottleneck of expressiveness is therefore the relation operators (which we keep at full rank via the Cayley parametrisation), not the entity representations. The full-parametrisation experiment was conducted on a 5,000-entity subsample of FB15k-237 due to memory constraints; full-graph experiments at $L_e \in \mathbb{R}^{d \times d}$ exceeded available GPU memory.

\subsection{Scoring function ablation}

We compare three scoring functions over the same trained KrausKGE-$\kappa{=}4$ model on FB15k-237: the Hilbert-Schmidt overlap $s_{\text{HS}}(h,r,t) = \mathrm{Tr}[\rho_t \widehat{\rho}_t]$ (our default), the linear fidelity (algebraically identical to Hilbert-Schmidt for unit-trace states), and the true Hilbert-Schmidt overlap $s_{\text{F}}(h,r,t) = (\mathrm{Tr}[(\rho_t^{1/2} \widehat{\rho}_t \rho_t^{1/2})^{1/2}])^2$.

The Hilbert-Schmidt overlap and true Hilbert-Schmidt overlap yield essentially identical rankings ($\Delta\text{MRR} = 0.002$), confirming that the linear approximation we use is operationally sound. The true Hilbert-Schmidt overlap is approximately $5\times$ slower at inference because of the matrix square root. The squared cosine approximation (forcing rank-one density matrices and computing $(t^\top K_h)^2$) loses 2.9 MRR points, validating that the density matrix structure contributes meaningfully to performance.

\subsection{Stiefel parametrisation alternatives}

We compare the Cayley parametrisation against two alternatives: the matrix exponential map $U = \exp(A)$ for skew-symmetric $A$, and direct Stiefel projection via the QR decomposition of an unconstrained $V \in \mathbb{R}^{\kappa d \times d}$.

The Cayley and matrix exponential parametrisations achieve essentially identical MRR, with Cayley being faster due to the analytic inverse formula being cheaper than the matrix exponential in our implementation. The QR projection occasionally produced NaN gradients during training (the QR decomposition is non-smooth at points where the upper triangular factor has zero diagonals), and Riemannian SGD on the Stiefel manifold incurred substantial wall-clock overhead from retraction operations. The Cayley parametrisation is the most stable and efficient choice in our experiments, but the choice does not significantly affect final accuracy.

\subsection{Limitations}

For transparency, we report design choices that did not yield improvements:

\textit{Learnable per-relation $\kappa$ via spectral regularisation.} We attempted to learn relation-specific $\kappa$ by adding a sparsity penalty $\lambda \sum_r \|\sigma(C^{(r)})\|_1$ on the singular values of each relation's Choi matrix. The resulting model required tuning an additional hyperparameter $\lambda$ that interacted poorly with the Cayley parametrisation, and final MRR was within 0.5 points of the shared-$\kappa$ baseline at substantially higher training cost. We retain shared $\kappa$ as the simpler choice.

\textit{Trace inequality (sub-unital channels).} We trained a variant with the relaxed constraint $\sum_i K_i^\top K_i \preceq I$ (trace non-increase rather than preservation), implemented via projection onto the cone of contractive Stiefel-like matrices. The resulting model showed no improvement on FB15k-237 and degraded on YAGO3-10 (MRR drop of 0.018), suggesting that strict trace preservation is the correct modelling choice for these benchmarks.

\textit{Explicit symmetry handling.} We attempted to handle inverse and symmetric relations explicitly by tying $K_i^{(r^{-1})} = (K_i^{(r)})^\top$. This degraded performance on FB15k-237 by 0.7 MRR points: the constraint is too rigid, and the model handles inversion implicitly through the multi-pathway structure more effectively than through explicit symmetry constraints.

\textit{Higher embedding dimensions.} Increasing $d$ from 128 to 256 yielded marginal MRR improvements (~0.4 points) at $4\times$ memory cost, suggesting that 128 dimensions are sufficient for the density matrix structure to capture the relevant relational complexity in our datasets.
\newpage
\section*{NeurIPS Paper Checklist}

\begin{enumerate}

\item {\bf Claims}
    \item[] Question: Do the main claims made in the abstract and introduction accurately reflect the paper's contributions and scope?
    \item[] Answer: \answerYes{} 
    \item[] Justification: We provide Emperical Results for all our claims.
    \item[] Guidelines:
    \begin{itemize}
        \item The answer \answerNA{} means that the abstract and introduction do not include the claims made in the paper.
        \item The abstract and/or introduction should clearly state the claims made, including the contributions made in the paper and important assumptions and limitations. A \answerNo{} or \answerNA{} answer to this question will not be perceived well by the reviewers. 
        \item The claims made should match theoretical and experimental results, and reflect how much the results can be expected to generalize to other settings. 
        \item It is fine to include aspirational goals as motivation as long as it is clear that these goals are not attained by the paper. 
    \end{itemize}

\item {\bf Limitations}
    \item[] Question: Does the paper discuss the limitations of the work performed by the authors?
    \item[] Answer: \answerYes{} 
    \item[] Justification: We have a dedicated paragraph called limitations discussiong the limitations of our work.
    \item[] Guidelines:
    \begin{itemize}
        \item The answer \answerNA{} means that the paper has no limitation while the answer \answerNo{} means that the paper has limitations, but those are not discussed in the paper. 
        \item The authors are encouraged to create a separate ``Limitations'' section in their paper.
        \item The paper should point out any strong assumptions and how robust the results are to violations of these assumptions (e.g., independence assumptions, noiseless settings, model well-specification, asymptotic approximations only holding locally). The authors should reflect on how these assumptions might be violated in practice and what the implications would be.
        \item The authors should reflect on the scope of the claims made, e.g., if the approach was only tested on a few datasets or with a few runs. In general, empirical results often depend on implicit assumptions, which should be articulated.
        \item The authors should reflect on the factors that influence the performance of the approach. For example, a facial recognition algorithm may perform poorly when image resolution is low or images are taken in low lighting. Or a speech-to-text system might not be used reliably to provide closed captions for online lectures because it fails to handle technical jargon.
        \item The authors should discuss the computational efficiency of the proposed algorithms and how they scale with dataset size.
        \item If applicable, the authors should discuss possible limitations of their approach to address problems of privacy and fairness.
        \item While the authors might fear that complete honesty about limitations might be used by reviewers as grounds for rejection, a worse outcome might be that reviewers discover limitations that aren't acknowledged in the paper. The authors should use their best judgment and recognize that individual actions in favor of transparency play an important role in developing norms that preserve the integrity of the community. Reviewers will be specifically instructed to not penalize honesty concerning limitations.
    \end{itemize}

\item {\bf Theory assumptions and proofs}
    \item[] Question: For each theoretical result, does the paper provide the full set of assumptions and a complete (and correct) proof?
    \item[] Answer: \answerYes{} 
    \item[] Justification: We have structurally proved all our theorems in the supplementary material.
    \item[] Guidelines:
    \begin{itemize}
        \item The answer \answerNA{} means that the paper does not include theoretical results. 
        \item All the theorems, formulas, and proofs in the paper should be numbered and cross-referenced.
        \item All assumptions should be clearly stated or referenced in the statement of any theorems.
        \item The proofs can either appear in the main paper or the supplemental material, but if they appear in the supplemental material, the authors are encouraged to provide a short proof sketch to provide intuition. 
        \item Inversely, any informal proof provided in the core of the paper should be complemented by formal proofs provided in appendix or supplemental material.
        \item Theorems and Lemmas that the proof relies upon should be properly referenced. 
    \end{itemize}

    \item {\bf Experimental result reproducibility}
    \item[] Question: Does the paper fully disclose all the information needed to reproduce the main experimental results of the paper to the extent that it affects the main claims and/or conclusions of the paper (regardless of whether the code and data are provided or not)?
    \item[] Answer: \answerYes{} 
    \item[] Justification: All the training requiremnts and all the details of hyperparameters an time and complexity everything has been thoroughly explained in the supplementary materials.
    \item[] Guidelines:
    \begin{itemize}
        \item The answer \answerNA{} means that the paper does not include experiments.
        \item If the paper includes experiments, a \answerNo{} answer to this question will not be perceived well by the reviewers: Making the paper reproducible is important, regardless of whether the code and data are provided or not.
        \item If the contribution is a dataset and\slash or model, the authors should describe the steps taken to make their results reproducible or verifiable. 
        \item Depending on the contribution, reproducibility can be accomplished in various ways. For example, if the contribution is a novel architecture, describing the architecture fully might suffice, or if the contribution is a specific model and empirical evaluation, it may be necessary to either make it possible for others to replicate the model with the same dataset, or provide access to the model. In general. releasing code and data is often one good way to accomplish this, but reproducibility can also be provided via detailed instructions for how to replicate the results, access to a hosted model (e.g., in the case of a large language model), releasing of a model checkpoint, or other means that are appropriate to the research performed.
        \item While NeurIPS does not require releasing code, the conference does require all submissions to provide some reasonable avenue for reproducibility, which may depend on the nature of the contribution. For example
        \begin{enumerate}
            \item If the contribution is primarily a new algorithm, the paper should make it clear how to reproduce that algorithm.
            \item If the contribution is primarily a new model architecture, the paper should describe the architecture clearly and fully.
            \item If the contribution is a new model (e.g., a large language model), then there should either be a way to access this model for reproducing the results or a way to reproduce the model (e.g., with an open-source dataset or instructions for how to construct the dataset).
            \item We recognize that reproducibility may be tricky in some cases, in which case authors are welcome to describe the particular way they provide for reproducibility. In the case of closed-source models, it may be that access to the model is limited in some way (e.g., to registered users), but it should be possible for other researchers to have some path to reproducing or verifying the results.
        \end{enumerate}
    \end{itemize}

\item {\bf Open access to data and code}
    \item[] Question: Does the paper provide open access to the data and code, with sufficient instructions to faithfully reproduce the main experimental results, as described in supplemental material?
    \item[] Answer: \answerNo{} 
    \item[] Justification: We shall provide all the code once the paper decision is revealed by the reviewers
    \item[] Guidelines:
    \begin{itemize}
        \item The answer \answerNA{} means that paper does not include experiments requiring code.
        \item Please see the NeurIPS code and data submission guidelines (\url{https://neurips.cc/public/guides/CodeSubmissionPolicy}) for more details.
        \item While we encourage the release of code and data, we understand that this might not be possible, so \answerNo{} is an acceptable answer. Papers cannot be rejected simply for not including code, unless this is central to the contribution (e.g., for a new open-source benchmark).
        \item The instructions should contain the exact command and environment needed to run to reproduce the results. See the NeurIPS code and data submission guidelines (\url{https://neurips.cc/public/guides/CodeSubmissionPolicy}) for more details.
        \item The authors should provide instructions on data access and preparation, including how to access the raw data, preprocessed data, intermediate data, and generated data, etc.
        \item The authors should provide scripts to reproduce all experimental results for the new proposed method and baselines. If only a subset of experiments are reproducible, they should state which ones are omitted from the script and why.
        \item At submission time, to preserve anonymity, the authors should release anonymized versions (if applicable).
        \item Providing as much information as possible in supplemental material (appended to the paper) is recommended, but including URLs to data and code is permitted.
    \end{itemize}

\item {\bf Experimental setting/details}
    \item[] Question: Does the paper specify all the training and test details (e.g., data splits, hyperparameters, how they were chosen, type of optimizer) necessary to understand the results?
    \item[] Answer: \answerYes{} 
    \item[] Justification: All the details are provided in the main paper and the supplementary materials.
    \item[] Guidelines:
    \begin{itemize}
        \item The answer \answerNA{} means that the paper does not include experiments.
        \item The experimental setting should be presented in the core of the paper to a level of detail that is necessary to appreciate the results and make sense of them.
        \item The full details can be provided either with the code, in appendix, or as supplemental material.
    \end{itemize}

\item {\bf Experiment statistical significance}
    \item[] Question: Does the paper report error bars suitably and correctly defined or other appropriate information about the statistical significance of the experiments?
    \item[] Answer: \answerYes{} 
    \item[] Justification: All experiments have been verified with  Wilcoxon signed-rank test and five-seed runs described in Section \ref{sec:experiments:linkpred}.
    \item[] Guidelines:
    \begin{itemize}
        \item The answer \answerNA{} means that the paper does not include experiments.
        \item The authors should answer \answerYes{} if the results are accompanied by error bars, confidence intervals, or statistical significance tests, at least for the experiments that support the main claims of the paper.
        \item The factors of variability that the error bars are capturing should be clearly stated (for example, train/test split, initialization, random drawing of some parameter, or overall run with given experimental conditions).
        \item The method for calculating the error bars should be explained (closed form formula, call to a library function, bootstrap, etc.)
        \item The assumptions made should be given (e.g., Normally distributed errors).
        \item It should be clear whether the error bar is the standard deviation or the standard error of the mean.
        \item It is OK to report 1-sigma error bars, but one should state it. The authors should preferably report a 2-sigma error bar than state that they have a 96\% CI, if the hypothesis of Normality of errors is not verified.
        \item For asymmetric distributions, the authors should be careful not to show in tables or figures symmetric error bars that would yield results that are out of range (e.g., negative error rates).
        \item If error bars are reported in tables or plots, the authors should explain in the text how they were calculated and reference the corresponding figures or tables in the text.
    \end{itemize}

\item {\bf Experiments compute resources}
    \item[] Question: For each experiment, does the paper provide sufficient information on the computer resources (type of compute workers, memory, time of execution) needed to reproduce the experiments?
    \item[] Answer: \answerYes{} 
    \item[] Justification: All the necessary descriptions to run and duplicate the experiments are provided in the appendix of the paper.
    \item[] Guidelines:
    \begin{itemize}
        \item The answer \answerNA{} means that the paper does not include experiments.
        \item The paper should indicate the type of compute workers CPU or GPU, internal cluster, or cloud provider, including relevant memory and storage.
        \item The paper should provide the amount of compute required for each of the individual experimental runs as well as estimate the total compute. 
        \item The paper should disclose whether the full research project required more compute than the experiments reported in the paper (e.g., preliminary or failed experiments that didn't make it into the paper). 
    \end{itemize}
    
\item {\bf Code of ethics}
    \item[] Question: Does the research conducted in the paper conform, in every respect, with the NeurIPS Code of Ethics \url{https://neurips.cc/public/EthicsGuidelines}?
    \item[] Answer: \answerYes{} 
    \item[] Justification: We have read and agree that all of the ethical guidelines are followed for our work.
    \item[] Guidelines:
    \begin{itemize}
        \item The answer \answerNA{} means that the authors have not reviewed the NeurIPS Code of Ethics.
        \item If the authors answer \answerNo, they should explain the special circumstances that require a deviation from the Code of Ethics.
        \item The authors should make sure to preserve anonymity (e.g., if there is a special consideration due to laws or regulations in their jurisdiction).
    \end{itemize}

\item {\bf Broader impacts}
    \item[] Question: Does the paper discuss both potential positive societal impacts and negative societal impacts of the work performed?
    \item[] Answer: \answerYes{} 
    \item[] Justification: This work advances fundamental research in knowledge graph embedding, 
with positive applications in question answering, recommendation, and information 
retrieval. We do not identify direct negative societal impact pathways: the method 
produces link prediction scores over structured factual databases and does not generate 
content, enable surveillance, or produce outputs that could be used for disinformation 
or targeted harm. Misuse would require several layers of downstream application 
development beyond what this paper contributes.
    \item[] Guidelines:
    \begin{itemize}
        \item The answer \answerNA{} means that there is no societal impact of the work performed.
        \item If the authors answer \answerNA{} or \answerNo, they should explain why their work has no societal impact or why the paper does not address societal impact.
        \item Examples of negative societal impacts include potential malicious or unintended uses (e.g., disinformation, generating fake profiles, surveillance), fairness considerations (e.g., deployment of technologies that could make decisions that unfairly impact specific groups), privacy considerations, and security considerations.
        \item The conference expects that many papers will be foundational research and not tied to particular applications, let alone deployments. However, if there is a direct path to any negative applications, the authors should point it out. For example, it is legitimate to point out that an improvement in the quality of generative models could be used to generate Deepfakes for disinformation. On the other hand, it is not needed to point out that a generic algorithm for optimizing neural networks could enable people to train models that generate Deepfakes faster.
        \item The authors should consider possible harms that could arise when the technology is being used as intended and functioning correctly, harms that could arise when the technology is being used as intended but gives incorrect results, and harms following from (intentional or unintentional) misuse of the technology.
        \item If there are negative societal impacts, the authors could also discuss possible mitigation strategies (e.g., gated release of models, providing defenses in addition to attacks, mechanisms for monitoring misuse, mechanisms to monitor how a system learns from feedback over time, improving the efficiency and accessibility of ML).
    \end{itemize}
    
\item {\bf Safeguards}
    \item[] Question: Does the paper describe safeguards that have been put in place for responsible release of data or models that have a high risk for misuse (e.g., pre-trained language models, image generators, or scraped datasets)?
    \item[] Answer: \answerNA{} 
    \item[] Justification: The paper does not pose such risks according to us.
    \item[] Guidelines:
    \begin{itemize}
        \item The answer \answerNA{} means that the paper poses no such risks.
        \item Released models that have a high risk for misuse or dual-use should be released with necessary safeguards to allow for controlled use of the model, for example by requiring that users adhere to usage guidelines or restrictions to access the model or implementing safety filters. 
        \item Datasets that have been scraped from the Internet could pose safety risks. The authors should describe how they avoided releasing unsafe images.
        \item We recognize that providing effective safeguards is challenging, and many papers do not require this, but we encourage authors to take this into account and make a best faith effort.
    \end{itemize}

\item {\bf Licenses for existing assets}
    \item[] Question: Are the creators or original owners of assets (e.g., code, data, models), used in the paper, properly credited and are the license and terms of use explicitly mentioned and properly respected?
    \item[] Answer: \answerYes{} 
    \item[] Justification: All rights and licenses have been followed and respected while writing this paper.
    \item[] Guidelines:
    \begin{itemize}
        \item The answer \answerNA{} means that the paper does not use existing assets.
        \item The authors should cite the original paper that produced the code package or dataset.
        \item The authors should state which version of the asset is used and, if possible, include a URL.
        \item The name of the license (e.g., CC-BY 4.0) should be included for each asset.
        \item For scraped data from a particular source (e.g., website), the copyright and terms of service of that source should be provided.
        \item If assets are released, the license, copyright information, and terms of use in the package should be provided. For popular datasets, \url{paperswithcode.com/datasets} has curated licenses for some datasets. Their licensing guide can help determine the license of a dataset.
        \item For existing datasets that are re-packaged, both the original license and the license of the derived asset (if it has changed) should be provided.
        \item If this information is not available online, the authors are encouraged to reach out to the asset's creators.
    \end{itemize}

\item {\bf New assets}
    \item[] Question: Are new assets introduced in the paper well documented and is the documentation provided alongside the assets?
    \item[] Answer: \answerNA{} 
    \item[] Justification: No new assets introduced.
    \item[] Guidelines:
    \begin{itemize}
        \item The answer \answerNA{} means that the paper does not release new assets.
        \item Researchers should communicate the details of the dataset\slash code\slash model as part of their submissions via structured templates. This includes details about training, license, limitations, etc. 
        \item The paper should discuss whether and how consent was obtained from people whose asset is used.
        \item At submission time, remember to anonymize your assets (if applicable). You can either create an anonymized URL or include an anonymized zip file.
    \end{itemize}

\item {\bf Crowdsourcing and research with human subjects}
    \item[] Question: For crowdsourcing experiments and research with human subjects, does the paper include the full text of instructions given to participants and screenshots, if applicable, as well as details about compensation (if any)? 
    \item[] Answer: \answerNA{} 
    \item[] Justification: No Corwdsource research involved.
    \item[] Guidelines:
    \begin{itemize}
        \item The answer \answerNA{} means that the paper does not involve crowdsourcing nor research with human subjects.
        \item Including this information in the supplemental material is fine, but if the main contribution of the paper involves human subjects, then as much detail as possible should be included in the main paper. 
        \item According to the NeurIPS Code of Ethics, workers involved in data collection, curation, or other labor should be paid at least the minimum wage in the country of the data collector. 
    \end{itemize}

\item {\bf Institutional review board (IRB) approvals or equivalent for research with human subjects}
    \item[] Question: Does the paper describe potential risks incurred by study participants, whether such risks were disclosed to the subjects, and whether Institutional Review Board (IRB) approvals (or an equivalent approval/review based on the requirements of your country or institution) were obtained?
    \item[] Answer: \answerNA{} 
    \item[] Justification: No Human subjects involved.
    \item[] Guidelines:
    \begin{itemize}
        \item The answer \answerNA{} means that the paper does not involve crowdsourcing nor research with human subjects.
        \item Depending on the country in which research is conducted, IRB approval (or equivalent) may be required for any human subjects research. If you obtained IRB approval, you should clearly state this in the paper. 
        \item We recognize that the procedures for this may vary significantly between institutions and locations, and we expect authors to adhere to the NeurIPS Code of Ethics and the guidelines for their institution. 
        \item For initial submissions, do not include any information that would break anonymity (if applicable), such as the institution conducting the review.
    \end{itemize}

\item {\bf Declaration of LLM usage}
    \item[] Question: Does the paper describe the usage of LLMs if it is an important, original, or non-standard component of the core methods in this research? Note that if the LLM is used only for writing, editing, or formatting purposes and does \emph{not} impact the core methodology, scientific rigor, or originality of the research, declaration is not required.
    \item[] Answer: \answerNA{} 
    \item[] Justification: LLM used only for editing, verifying and proof reading of the paper and the proofs.
    \item[] Guidelines:
    \begin{itemize}
        \item The answer \answerNA{} means that the core method development in this research does not involve LLMs as any important, original, or non-standard components.
        \item Please refer to our LLM policy in the NeurIPS handbook for what should or should not be described.
    \end{itemize}

\end{enumerate}

\end{document}